\documentclass{tlp}

\setcitestyle{numbers,square, citesep=comma}
\makeatletter
\def\@j@urnal{Theory and Practice of Logic Programming}
\renewcommand{\@biblabel}[1]{[#1]}
\makeatother

\journal{Theory and Practice of Logic Programming}

\usepackage{algorithm}
\usepackage{verbatim}
\usepackage{algpseudocode}
\usepackage{tikz}
\usepackage{amsmath}
\usepackage{graphicx}
\usepackage{amssymb}
\usepackage{enumitem}
\usepackage{multirow}
\usepackage{hyperref}   
\usepackage{subcaption}
\usepackage{color}
\newtheorem{theorem}{Theorem}[section]
\newtheorem{corollary}{Corollary}[theorem]
\newtheorem{lemma}[theorem]{Lemma}
\usetikzlibrary{positioning, arrows.meta, calc, shapes}
\usepackage{pgf}
\usetikzlibrary{arrows,shapes.arrows,shapes.geometric,shapes.multipart, decorations.pathmorphing,positioning,shapes.swigs,}

\newcommand{\myparagraph}[1]{\vspace{0.5em}\noindent\textbf{#1.}}

\begin{document}

\lefttitle{Habib et al.}

\jnlPage{1}{8}
\jnlDoiYr{2026}
\doival{10.1017/xxxxx}

\title[Single World Intervention Programs]{Efficient Counterfactual Reasoning in ProbLog via Single World Intervention Programs}%\thanks{Code available at https://github.com/saibib/swip }}

\begin{authgrp}
\author{\sn{Saimun} \gn{Habib}*$\dagger$}
\email{S.Habib-1@ed.ac.uk}
\author{\sn{Vaishak} \gn{Belle}*}
\email{vbelle@ed.ac.uk}
\author{\sn{Fengxiang} \gn{He}*}
\email{F.He@ed.ac.uk}

\affiliation{*University of Edinburgh, Edinburgh, United Kingdom}
\affiliation{$\dagger$ The MITRE Corporation, McLean, Virgina, USA\footnote{ Affiliation with The MITRE Corporation is provided for identification purposes only, and is not intended to convey or imply MITRE’s concurrence with, or support for, the positions, opinions, or viewpoints expressed by the author.}}

\end{authgrp}

\history{\sub{xx xx xxxx;} \rev{xx xx xxxx;} \acc{xx xx xxxx}}

\maketitle

\begin{abstract}
Probabilistic Logic Programming (PLP) languages, like ProbLog, naturally support reasoning under uncertainty, while maintaining a declarative and interpretable framework. Meanwhile, counterfactual reasoning (i.e., answering ``what  if'' questions) is critical for ensuring AI systems are robust and trustworthy; however, integrating this capability into PLP can be computationally prohibitive and unstable in accuracy. This paper addresses this challenge, by proposing an efficient program transformation for counterfactuals as Single World Intervention Programs (SWIPs) in ProbLog. By systematically splitting ProbLog clauses to observed and fixed components relevant to a counterfactual, we create a transformed program that (1) does not asymptotically exceed the computational complexity of existing methods, and is strictly smaller in common cases, and (2) reduces counterfactual reasoning to marginal inference over a simpler program. We formally prove the correctness of our approach,  which relies on a weaker set independence assumptions and is consistent with conditional independencies, showing the resulting marginal probabilities match the counterfactual distributions of the underlying Structural Causal Model in wide domains. Our method achieves a 35\% reduction in inference time versus existing methods in extensive experiments. This work makes complex counterfactual reasoning more computationally tractable and reliable, providing a crucial step towards developing more robust and explainable AI systems. The code is at \url{https://github.com/EVIEHub/swip}.
\end{abstract}


\section{Introduction}
\label{sec:intro}
Among causal queries, counterfactual questions of the form “What if $X$ had been different?” are crucial for explanation, diagnosis, credit assignment, and decision making \cite{pearl_book_2018, halpern_actual_2019}. To situate its value, it is helpful to distinguish probabilistic, causal, and counterfactual modeling along Pearl’s “ladder of causality”  \cite{pearl_causality_2009}. While probabilistic models describe associations, Structural Causal Models (SCMs) additionally specify directed functional relationships between variables with exogenous noise and how these relationships change under interventions \cite{pearl_causality_2009}. This hierarchy of associational, interventional, and counterfactual queries captures increasingly expressive forms of reasoning, and in particular, the capacity for counterfactual reasoning is a critical capability for artificial intelligence systems to discern effects of alternative choices by explicitly relating the actual world to hypothetical variants of a causal model’s structure under intervention to explain the observation(s) produced by the underlying generative mechanisms in either world \cite{gerstenberg_counterfactual_2024,epstude_functional}.
\newline \newline
\myparagraph{Why PLP?} \\
Counterfactuals in SCMs traditionally followed a procedure of abduction, action, and prediction \cite{pearl_causality_2009}, but Hopkins and Pearl argue SCMs become unwieldy in realistic domains as they lack first-order expressivity over entity and attribute relationals, quantified rules, and temporally structured actions \cite{hopkins_causality_2007}. For example, the classic “two-riflemen” scenario \cite{pearl_causality_2009} requires a model that distinguishes multiple agents, possible misfires, and a disjunctive causal mechanism. A propositional encoding must enumerate each marksman $A$ and $B$ taking aim at a target and whether they hit the target separately, whereas a relational encoding expresses the causal mechanism compactly: $\texttt{hits(Target) :- fired(Soldier,Target).}$
The broader lesson is that causal models benefit from symbolic structure, and symbolic systems benefit from principled causal semantics \cite{hopkins_causality_2007}. Probabilistic Logic Programming (PLP) sits naturally at this intersection. It provides explicit \emph{symbolic mechanisms} as interpretable rules describing how entities and relations interact, combined with probabilistic uncertainty. Such rules (e.g., \texttt{causes(Smoking,Cancer)}, \texttt{infects(P1,P2)}) support recursion, quantification, and relational generalization \cite{vennekens_cp-logic_2009, vennekens_embracing_2010}. However, strictly speaking PLP languages were not designed for causal inference \cite{pearl_causality_2009, vennekens_embracing_2010, vennekens_cp-logic_2009, DeRaedt.etal2007}. These languages lack a $do$-operator, formal intervention semantics, and tools for causal identification.
A simple ProbLog program illustrates this gap:
\begin{verbatim}
0.3::lifestyle(alice).
0.3::smokes(alice).
0.6::genetic_risk(alice).
cancer(alice) :- smokes(alice), genetic_risk(alice).
\end{verbatim}
Under distribution semantics, this program defines a joint probability over $\texttt{smokes(alice)}, \texttt{genetic\_risk(alice)}$, and $\texttt{cancer(alice)}$ with logical dependencies between them. Conditioning on \texttt{smokes(alice)=false} updates beliefs on Alice’s likelihood of cancer but does not implement the causal intervention $\do(\texttt{smokes(alice)}{=}false)$, which requires deleting the generative mechanism between smoking and cancer. This difference between conditioning and intervening marks the conceptual boundary between probabilistic and causal interpretation.

Recognizing this limitation, several PLP languages extend logic programming with causal meaning. Logic Program Annotated Disjunctions (LPADs), \cite{vennekens_logic_2004}, Causal Probabilistic-logic (CP-logic) \cite{vennekens_cp-logic_2009}, and ProbLog \cite{DeRaedt.etal2007} introduce probabilistic choices, and CP-logic interprets rules as probabilistic causal laws. Its intervention semantics disable or modify such laws, aligning CP-logic with SCM-style reasoning. However, computing counterfactuals still follows the abduction, action, prediction procedure and requires storing the full posterior $P(\mathbf{U} \mid \mathbf{E}=\mathbf{e})$ over exogenous causes $\mathbf{U}$ given the evidence or observations $\mathbf{e}$ .

To address this, another line of work adapts the \emph{Twin Network} construction from SCMs to ProbLog \cite{balke_probabilistic_2022, kiesel_what_2023}. Counterfactual inference is reduced to ordinary probabilistic inference by duplicating the program into factual and counterfactual copies linked by shared exogenous variables avoiding storing $P(U \mid E=e)$ and doing intervention and inference at once to compute $P(Y|do(X=x), E=e)$. While convenient, this approach suffers from (i) exponentially many cross-world independence assumptions \cite{richardson_single_2013,Shpitser.etal2021}, and (ii) demonstrable failures of these assumptions in important causal structures. In a practical sense, it is also limited for first-order relational models as it doubles program size, increasing compilation cost.

\subsection{Our Contributions}
In this work, we introduce \emph{Single-World Intervention Programs (SWIPs)}, a new method for counterfactual reasoning in ProbLog inspired by the Single-World Intervention Graph (SWIG) framework \cite{richardson_single_2013}. Rather than duplicating the model, SWIPs perform a semantics-preserving transformation of the program itself.

Returning to the earlier example, to compute $P(\texttt{cancer(alice)} \mid \do(\texttt{smokes(alice)}{=}0))$, SWIP removes the probabilistic fact \texttt{0.3::smokes(alice).}, inserts the deterministic fact \texttt{smokes(alice)=false}., and manipulates the rule for \texttt{cancer(alice).} by removing its dependence on the original smoking mechanism while making it reliant on the intervention.
More generally, given $\do(X{=}x)$, the SWIP transformation deletes all clauses defining $\mathbf{X}$,
inserts the deterministic fact(s) asserting $\mathbf{X} := \mathbf{x}$ in place of $\mathbf{X}$, and
eliminates redundant or unreachable rules through structural simplification.

The resulting SWIP is a simplified ProbLog program whose distribution matches the counterfactual semantics of the corresponding SCM. SWIPs offer practical advantages. Across extensive synthetic experiments, SWIPs produce significantly smaller unfolded programs and reduce compilation and inference time by approximately 35\% relative to Twin Networks. Since SWIPs simplify rather than duplicate rule structure, knowledge-compilation backends (d-DNNF, \texttt{SHARPSAT}) exploit the reduced treewidth directly \cite{eiter_treewidth-aware_2021}. We prove that (1) for any SCM encodable in ProbLog, SWIPs reproduce exactly the interventional and counterfactual distributions of the SCM under standard assumptions of unique supported models and faithful SCM encodings; (2) the SWIP approach is computationally less expensive than the Twin Network approach; and (3)
the grounded SWIP semantics coincide with CP-logic’s intervention semantics, unifying event-based and equation-based causal interpretations.

To our knowledge, SWIPs provide the first single-world, SCM-faithful counterfactual semantics inside ProbLog that avoid cross-world assumptions while retaining the expressive relational structure of logic programming.

\section{Preliminaries}
\label{sec:prelim}

\subsection{Causal Models and Counterfactuals}
\label{subsec:causal_models}

%Our work is grounded in the structural account of causality \cite{pearl_causality_2009}. 
A Structural Causal Model (SCM) $\mathcal M$ is defined as a tuple $\langle \mathbf{U},\mathbf{V}, \mathcal F\rangle$, where $\mathbf U$ are mutually independent exogenous variables, $\mathbf V$ are endogenous variables, and $\mathcal F = \{f_V\}_{V \in \mathbf{V}}$ is a set of structural equations of the form
$
V := f_V(\text{pa}(V), U_V),
$
with $\text{pa}(V)\subseteq \mathbf{V}$ and $U_V\subseteq \mathbf{U}$. Each SCM induces a directed acyclic graph (DAG) $\mathcal G$ encoding the causal dependencies among variables \cite{pearl_causality_2009}. Figure \ref{fig:generic} depicts a causal model between genetics $G$, lifestyle $L$, diet $D$, and health $H$ while omitting the implicit exogenous noise variables for simplicity.

% A Structural Causal Model (SCM) $\mathcal M$ can be represented by a tuple $\langle \mathbf{U},\mathbf{V}, \mathcal F\rangle$, where $\textbf U$ is a set of mutually independent exogenous variables, $\textbf V$ is a set of endogenous variables, and $\mathcal F = \{f_V\}_{V \in \textbf{V}}$ is a set of structural functions. These functions define a deterministic equations for each endogenous variable $V\in \textbf{V}$ as:
% \begin{equation*}
% V := f_V(\text{pa}(V), U_V)
% \end{equation*}
% where $pa(V)\subseteq \textbf{V}$ is the set of parents of $V$ and $U_V \subseteq \textbf{U}$ is the exogenous variables that directly influence $V$. Together the functions in $\mathcal{F}$ deterministically map any assignment $\mathbf u\in\mathbf{U}$ of exogenous variables to a unique assignment $\mathbf v\in\mathbf{V}$ of endogenous variables. Hence, a distribution $P(\mathbf u)$ over $\mathbf U$ is pushed forward, ie. induces, a distribution $P(\mathbf v)$ over $\mathbf V$ via these structural equations. We assume the model is well-posed, meaning the system of equations has a unique solution for every instantiation $u$ of $U$ \cite{pearl_causality_2009}. Each SCM $\mathcal M$ is associated with a directed acyclic graph (DAG) $\mathcal G$, where nodes represent variables in $\textbf{V}$ and directed edges point from parent to child.

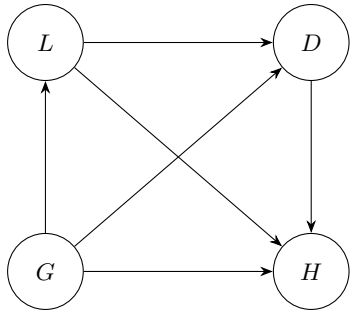
\begin{figure}[t]
\centering
\begin{tikzpicture}[
    node distance=2cm and 2.5cm,
    every node/.style={draw, circle, minimum size=1cm},
    ->, >=Stealth
]

% Nodes
\node (L) {$L$};
\node[right=of L] (D) {$D$};
\node[below=of D] (H) {$H$};
\node[below=of L] (G) {$G$};

% Edges
\draw (L) -- (D);
\draw (L) -- (H);
\draw (G) -- (D);
\draw (G) -- (H);
\draw (G) -- (L);
\draw (D) -- (H);

\end{tikzpicture}
\caption{A causal diagram with variables $L$, $D$, $G$, and $H$. The corresponding SCM declares $H$ as a function of $A$, $B$ and $C$,  $D$ as a function of $G$ and $L$, and $L$ as functions of $G$.}
\label{fig:generic}
\end{figure}

An intervention replaces the equations for a subset $\mathbf{X}\subseteq \mathbf{V}$ with constant assignments $\mathbf{X}:=\mathbf{x}$, producing a modified model $\mathcal{M}_{\mathbf{x}}$. The induced post-interventional distribution is
$
P(\mathbf{v} \mid do(\mathbf{X}:=\mathbf{x})) = 
\prod_{V_i \in \mathbf{V}\setminus\textbf{X}} P(v_i \mid \text{pa}(V_i))\cdot \mathbb{I}[x] \quad \text{in } \mathcal{M}_{\mathbf{x}}.
$
% Pearl’s do-calculus provides three transformation rules for identifying causal effects by adding, removing, or exchanging actions and observations under conditional-independence constraints derived from $\mathcal{G}$ \cite{pearl_causality_2009}. Detailed formulations of these rules are provided in Appendix~\ref{app:do_calculus} \sh{add appdx}.
Counterfactual queries, $P(Y \mid do(\mathbf{X}:=\mathbf{x}), \mathbf{E}{=}\mathbf{e})$, ask about outcomes under hypothetical interventions given observed evidence. Their computation typically follows the abduction, action, and prediction steps of \cite{balke_probabilistic_2022}. \emph{Abduction} requires storing the entire distribution of $P_{\mathcal{M}}(\textbf{U}|\textbf{E=e})$, while \emph{action} manipulates the model $\mathcal{M}$ into the intervened distribution $\mathcal{M}_\mathbf{x}$ for some assignment $\textbf{X}:=\textbf{x}$, and \emph{prediction} finally calculates $P_{\mathcal{M_{\mathbf{x}}}}(\mathbf{Y}|\mathbf{E}=\mathbf{e})$ \cite{pearl_causality_2009,balke_probabilistic_2022,kiesel_what_2023}, but the twin network circumvents the abduction step and combines action and prediction by taking $\mathcal{M}$ and creating $\mathcal{M}^K$ given by the tuple $\langle \textbf{U}, \textbf{V}' \cup \textbf{V}, \mathcal{F} \rangle$. It sets $\textbf{V}'=\textbf{V}$ and uses it to create a factual and counterfactual set of equations with shared exogeneity defined as 

\begin{equation*}
X := \begin{cases}
f_X(pa(X), U_X), \quad\,\, X \in \textbf{V} \\
f_X(pa(X)', U_X), \quad X \in \textbf{V}'
\end{cases}
\end{equation*} 
where $pa(X)' = \{X'|X\in pa(X)\}$. Interventions are set on the variables in $\textbf{V}'$, ie. $\textbf{X}' := \textbf{x}$ and we have:
\begin{equation*}
P_{\mathcal{M}}(\,\cdot\,|\textbf{E}=\textbf{e}, do(\textbf{X}:=\textbf{x})) = 
% P_{\mathcal{M}_\textbf{x}}(\,\cdot\,|\textbf{E}=\textbf{e}) = 
P_{\mathcal{M}_\textbf{x}^K}(\,\cdot\,|\textbf{E}=\textbf{e})
\end{equation*}

Alternatively, the Single-World Intervention Graph (SWIG) formalism provides a convenient graphical encoding of the abduction, action, and prediction steps used to evaluate counterfactuals \cite{richardson_single_2013, pearl_causality_2009}. 
Given a graph $\mathcal{G}$ encoding an SCM $\mathcal{M}$ and an intervention $do(\textbf{X}:=\textbf{x})$, a SWIG is obtained by splitting each intervened node $X \in \textbf{X}$ into two: a factual copy $X$ that receives incoming edges from $pa(X)$ and a fixed counterfactual copy $X':=x$ with outgoing edges to its original children. Exogenous variables remain as is, yielding a single graph that represents factual and counterfactual variables within one world while avoiding ad hoc duplication of independent noise sources. Following standard graphical separation criteria, the SWIG is a complete independence oracle and determines which counterfactual queries are identified by observed data \cite{richardson_single_2013}.

\subsection{ProbLog}
\label{subsec:ProbLog_sem}
ProbLog is a probabilistic logic programming language %that extends Prolog 
with distribution semantics \cite{DeRaedt.etal2007}. A ProbLog program $\mathcal{P}$ is a pair $(\mathrm{LP}(\mathcal{P}),\text{Facts}(\mathcal{P}))$. $\text{Facts}(\mathcal{P})$ is a set of probabilistic facts of the form $\pi::a$, where $a$ is a ground atom from the set of external propositions $\mathfrak{E(B)}$ and $\pi  \in [0,1]$ is its probability. These facts are assumed to be mutually independent and correspond to the exogenous variables $\textbf{U}$ in an SCM. $\mathrm{LP}(\mathcal{P})$ is a set of logical clauses (rules) of the form $h\leftarrow b_1,\dots,b_n$, where $h$ is an internal proposition from $\mathfrak{\mathfrak{I(P)}}$ and the body $\{b_1,\dots,b_n\}$ is a set of literals. These rules correspond to the structural equations $\mathcal{F}$ in an SCM. 

The semantics of $\mathcal{P}$ define a probability distribution $\pi_{\mathcal{P}}$ over possible worlds $\omega$, where a world is a complete truth assignment to all ground atoms in $\mathfrak{B} = \mathfrak{J}(\mathfrak{B}) \cup \mathfrak{E}(\mathfrak{B})$. 
Formally, the probability of a world $\omega$ is
$
\pi_{\mathcal{P}}(\omega)
= \prod_{a_i \in \omega} \pi_i
  \prod_{a_i \notin \omega} \big(1 - \pi_i\big)
$
for all probabilistic facts $a_i \in Facts(P)$, i.e., the product of the probabilities of all true probabilistic facts and the complements of those that are false. The probability of a query $\phi$ is the sum of the probabilities of all worlds in which $\phi$ is true\cite{DeRaedt.etal2007}: 
$$\pi_{\mathcal{P}}(\phi)=\sum_{\omega \models	\phi}\pi_{\mathcal{P}}(\omega).$$
For a ProbLog program to correctly represent an SCM, it must have unique supported models, meaning that for any truth assignment to the external propositions $\mathfrak{E(P)}$, the logical rules in $\mathrm{LP}(\mathcal{P})$ must yield a single, unique truth assignment for all internal propositions $\mathfrak{\mathfrak{I(P)}}$. A sufficient condition for this property is that the dependency graph of $\mathrm{LP}(\mathcal{P})$ is acyclic \cite{Kiesel.etal2023}.

Within ProbLog, Pearl's d-separation reasoning for independence and identifiability can be implemented declaratively as a meta-interpreter that encodes both the syntactic rules of do-calculus and the associated independence checks as higher-order logic predicates \cite{ruckschlos_subtlety_nodate}. This enables automated symbolic reasoning about causal identifiability directly within the ProbLog environment. 

\section{Challenges of the Twin Network in ProbLog}
\label{sec:tn}

This section analyses a prevailing method for computing counterfactuals in ProbLog through a program transformation that constructs a Twin Network \cite{Kiesel.etal2023}. This approach operationalizes the SCM framework by creating a new ProbLog program, $\mathcal{T}(\mathcal{P})$, that explicitly represents both a factual and a counterfactual world. The transformation systematically duplicates all internal propositions $\mathfrak{I(P)}$ and their defining rules in $\mathrm{LP}(\mathcal{P})$, while the external propositions in $\text{Facts}(\mathcal{P})$ remain shared between the two worlds. This sharing of exogenous variables is the mechanism that relates the factual and counterfactual outcomes and the full procedure is detailed in Algorithm \ref{alg:twin_construct}. 

This approach, however, implicitly relies on a strong and untestable cross-world assumption of all of these exogenous variables being independent of one another and in fact, the number of assumptions required grows at a doubly exponential rate \cite{Shpitser.etal2021,richardson_single_2013}. The duplication of the entire program by nature introduces higher compute cost for program compilation in proportion to the the program size and length of clauses in the program. In the worst case, it is $\Theta(|\mathcal{P}|\cdot L_{\max})$
\begin{figure}[h]
\centering

% -----------------------------
%   (a) Original causal model
% -----------------------------
\begin{subfigure}[t]{0.48\textwidth}
\centering
\begin{tikzpicture}[
    every node/.style={draw, circle, minimum size=.5cm, font=\bfseries},
    ->, >=Stealth
]

% Node positions
\node (uL) at (1,4.5) {$\epsilon_L$};
\node (L)  at (1,2.5) {$L$};

\node (uG) at (3,6.5) {$\epsilon_G$};
\node (G)  at (3,4.5) {$G$};

\node (uD) at (5,4.5) {$\epsilon_D$};
\node (D)  at (5,2.5) {$D$};

\node (uH) at (3,0) {$\epsilon_H$};
\node (H)  at (3,1.5) {$H$};

% Edges
\draw (uL) -- (L);
\draw (uD) -- (D);
\draw (uG) -- (G);
\draw (uH) -- (H);

\draw (G) -- (L);
\draw (G) -- (D);
\draw (G) -- (H);
\draw (L) -- (H);
\draw (D) -- (H);
\draw (L) -- (D);

\end{tikzpicture}
\caption{Original causal model}
\end{subfigure}
\hspace{-0.5cm} % Shift subfigure (b) slightly left to avoid margin creep
% -----------------------------
%   (b) Twin Network
% -----------------------------
\begin{subfigure}[t]{0.48\textwidth}
\centering
\begin{tikzpicture}[
    every node/.style={draw, circle, minimum size=.5cm, font=\bfseries},
    ->, >=Stealth
]

% Noise variables
\node (uG2) at (4,6.5) {$\epsilon_G$};
\node (uL2) at (4,4.5) {$\epsilon_L$};
\node (uD2) at (4,2.5) {$\epsilon_D$};
\node (uH2) at (4,0.5) {$\epsilon_H$};

% Factual nodes
\node (G2) at (2,6) {$G$};
\node (L2) at (1,4) {$L$};
\node (D2) at (3,2) {$D$};
\node (H2) at (2,0) {$H$};

% Counterfactual nodes
\node (Gstar2) at (6,6) {$G^*$};
\node (Lstar2) at (7,4) {$L^*$};
\node (Dstar2) at (5,2) {$D^*$};
\node (Hstar2) at (6,0) {$H^*$};

% Edges (noise → factual)
\draw (uG2) -- (G2);
\draw (uL2) -- (L2);
\draw (uD2) -- (D2);
\draw (uH2) -- (H2);

% Edges (factual → factual)
\draw (G2) -- (L2);
\draw (G2) -- (D2);
\draw (G2) -- (H2);
\draw (L2) -- (H2);
\draw (D2) -- (H2);
\draw (L2) -- (D2);

% Noise → counterfactual
\draw (uG2) -- (Gstar2);
\draw (uL2) -- (Lstar2);
\draw (uD2) -- (Dstar2);
\draw (uH2) -- (Hstar2);

% Counterfactual dependencies
\draw (Gstar2) -- (Lstar2);
\draw (Gstar2) -- (Dstar2);
\draw (Gstar2) -- (Hstar2);
\draw (Lstar2) -- (Hstar2);
\draw (Dstar2) -- (Hstar2);
\draw (Lstar2) -- (Dstar2);

\end{tikzpicture}
\caption{Twin network with counterfactual variables}
\end{subfigure}

\caption{Side-by-side comparison of (a) the original structural causal model and (b) its corresponding twin network construction.}
\label{fig:twin}
\end{figure}

\begin{theorem}[Twin Network Transformation Complexity]
Let $\mathcal{P} = (\mathrm{LP}(\mathcal{P}), \mathrm{Facts}(\mathcal{P}))$ be a ProbLog program with $|\mathcal{P}| = |\mathrm{LP}(\mathcal{P})| + |\mathrm{Facts}(\mathcal{P})|$ denoting the total number of clauses and facts. Let $L_{\max}$ be the maximum body length of any clause in $\mathrm{LP}(\mathcal{P})$. Then the Twin Network transformation $\mathcal T(\mathcal{P})$ following Algorithm 2 from Kiesel et al. (2023) has complexity $\Theta(|\mathcal{P}| \cdot L_{\max})$.
\label{thrm:tn_complexity}
\end{theorem}

An intervention $do(\textbf{X}:=\textbf{x})$ is applied to the program by modifying the rules in the counterfactual part of the program. Subsequently, a counterfactual query $P_{\mathcal{M}^\mathcal{K}_\mathbf{x}}(Y|\textbf{E}=\textbf{e})$ is evaluated by computing a standard marginal probability on the transformed program $\mathcal{T}(\mathcal{P})$, as shown in Algorithm \ref{alg:twin_query}. The inference complexity itself is determined by an algorithm, such as knowledge compilation, whose runtime is given by a function $g(w(G))$ for a program with primal graph $G$ of treewidth given by $w(G)$ \cite{eiter_treewidth-aware_2021}. A necessary condition for the validity of this approach for a counterfactual queries is the original program $P$ has unique supported models, for which a sufficient condition is an acyclic underlying logic program $\mathrm{LP}(\mathcal{P})$.

\begin{algorithm}[H]
\caption{\textsc{ConstructTwinNetwork}$(\mathcal{P},do(\textbf{X}:=\textbf{x}))$}
\label{alg:twin_construct}
\begin{algorithmic}[1]
\Require ProbLog program $\mathcal{P}=(\mathrm{LP}(\mathcal{P}),\text{Facts}(\mathcal{P}))$; intervention $\textbf{X}:=\textbf{x}$
\Ensure Twin Network program $\mathcal{T}(\mathcal{P})$
\State Initialize $\mathcal{T}(\mathcal{P})\leftarrow \emptyset$
\ForAll{$p::a\in \text{Facts}(\mathcal{P})$}
    \If{$a\in \textbf{X}$}
        \State Add $1.0::a$ and $0.0::a'$ to $\mathcal{T}(\mathcal{P})$
    \Else
        \State Add $p::a$ and $p::a'$ to $\mathcal{T}(\mathcal{P})$
    \EndIf
\EndFor
\ForAll{$h\leftarrow b_1,\dots ,b_n \in \mathrm{LP}(\mathcal{P})$}
    \If{$h\not\in \textbf{X}$}
        \State Add $h\leftarrow b_1,\dots ,b_n$ to $\mathcal{T}(\mathcal{P})$
        \State Add $h'\leftarrow b_1',\dots ,b_n'$ to $\mathcal{T}(\mathcal{P})$
    \EndIf
\EndFor
\State \Return $\mathcal{T}(\mathcal{P})$
\end{algorithmic}
\end{algorithm}

\begin{algorithm}[H]
\caption{\textsc{EvaluateTwinNetworkQuery}$(\mathcal{P},\mathcal{T}(\mathcal{P}),\textbf{X}:=\textbf{x},\textbf{E} = \textbf{e},\phi )$}
\label{alg:twin_query}
\begin{algorithmic}[1]
    \Require Original program $\mathcal{P}$, Twin Network $\mathcal{T}(\mathcal{P})$, intervention $\textbf{X}:=\textbf{x}$, evidence $\textbf{E} = \textbf{e}$, query formula $\phi $
    \Ensure Counterfactual probability $\pi_{\mathcal{P}}(\phi_{\textbf{x}}|\textbf{E} = \textbf{e})$
    \State $p_1\leftarrow \pi_{\mathcal{P}}(\textbf{E} = \textbf{e})$
    \State $p_2\leftarrow \pi_{\mathcal{T}(\mathcal{P})}(\phi'\wedge\textbf{E} = \textbf{e})$
    \State \Return $p_2/p_1$
\end{algorithmic}
\end{algorithm}

\section{Single World Intervention Programs}
\label{sec:swip}
To address both limitations, we introduce our  \emph{Single-World Intervention Fact Transformation} (SWIFT) algorithm detailed in Algorithm \ref{alg:swip_transform} to produce a Single World Intervention Program (SWIP) fit for a counterfactual query given an intervention and evidence. 

The SWIFT algorithm operationalizes the "graph surgery" of Single-World Intervention Graphs \cite{richardson_single_2013}, shown in Figure \ref{fig:swig}, at the level of logical rules. Because rules interact through relational, quantified, and recursive dependencies, altering or deleting a clause may change the support, reachability, and logical structure of downstream atoms. The Twin Network approach avoids this by model duplication, at the cost of additional assumptions and computational overhead. Instead, the SWIFT algorithm provides a rule rewriting procedure that preserves the deterministic closure implied by existing rules and unique supported model requirement by the distribution semantics to produce a program semantically equivalent to the intervened Structural Causal Model.

% \begin{figure}[h]
%     \centering
%     \begin{tikzpicture}[
%     node distance=2cm and 2.5cm,
%     every node/.style={draw, ellipse, minimum size=1.5cm},
%     ->, >=Stealth
%     ]
        
%         % Nodes
%         \node[name=L,shape=swig vsplit,
%             swig vsplit={gap=3pt}]{
%             \nodepart{left}{$L$}
%             \nodepart{right}{$L=l$} };
%         \node[right=of L] (D) {$D(L=l)$};
%         \node[below=of L] (G) {$G$};
%         \node[below=of D] (H) {$H(L=l)$};
        
%         % Edges
%         \draw (G) to[out=90,in=180] (L);
%         \draw (G) -- (D);
%         \draw (G) -- (H);
%         \draw (L) -- (H);
%         \draw (D) -- (H);
%         \draw (L) -- (D);

%     \end{tikzpicture}
    
%     \caption{SWIG where intervention $fix(L=l)$ is acted}
%     \label{fig:swig}
% \end{figure}

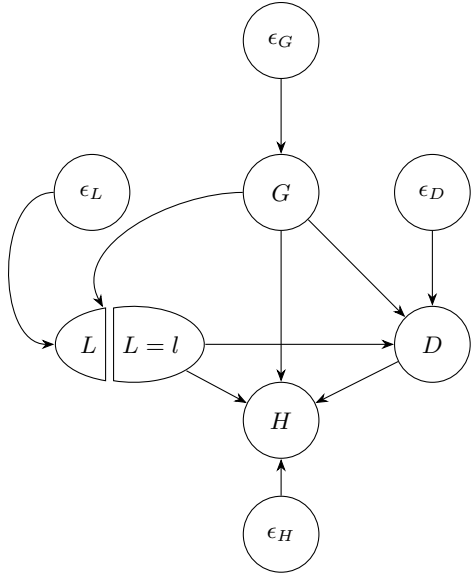
\begin{figure}[h]
    \centering
\begin{tikzpicture}[
    every node/.style={draw, circle, minimum size=1cm, font=\bfseries},
    ->, >=Stealth
]

% Node positions
\node (uL) at (.5,4.5) {$\epsilon_L$};
% \node (L)  at (1,2.5) {$L$};
\node[
    name=L,
    shape=swig vsplit,
    swig vsplit={gap=3pt}
] at (1,2.5) {
    \nodepart{left}{$L$}
    \nodepart{right}{$L=l$}
};
\node (uG) at (3,6.5) {$\epsilon_G$};
\node (G)  at (3,4.5) {$G$};

\node (uD) at (5,4.5) {$\epsilon_D$};
\node (D)  at (5,2.5) {$D$};

\node (uH) at (3,0) {$\epsilon_H$};
\node (H)  at (3,1.5) {$H$};

% Edges
\draw (uL) to[out=180,in=180] (L);
\draw (uD) -- (D);
\draw (uG) -- (G);
\draw (uH) -- (H);

\draw (G) to[out=180,in=120] (L);
\draw (G) -- (D);
\draw (G) -- (H);
\draw (L) -- (H);
\draw (D) -- (H);
\draw (L) -- (D);

\end{tikzpicture}
 \caption{SWIG where intervention $fix(L=l)$ is acted}
    \label{fig:swig}
\end{figure}

To distinguish interventions in the Twin Network setting and the SWIP setting, we use the $fix(\textbf{X}:=\textbf{x})$ notation. This procedure first removes all clauses from $\mathrm{LP}(\mathcal{P})$ that define the intervened propositions in $\textbf{X}$. This step corresponds to severing the causal arrows into the intervened nodes. Second, it iterates through the remaining rules and replaces any occurrence of an intervened atom $X_i\in \textbf{X}$ in a rule body with a new, unique atom $X_{i,fixed}$ that represents its fixed value. This corresponds to redirecting the outgoing causal arrows from the original random node to the new fixed-value node. Finally, it adds deterministic facts to assert the values of these new fixed atoms. In contrast to constructing $\mathcal{T}(\mathcal{P})$, the complexity of SWIFT is at worst proportional to the program size, ie. when an intervened variable appears as an atom in every rule, while for $\mathcal{T}(\mathcal{P})$, it is always proportional. The resulting SWIP, $\mathcal{S}(\mathcal{P})$, is a valid ProbLog program that directly represents the counterfactual world.

\begin{theorem}[SWIP Transformation Complexity]
    Let $\mathcal{P} = (\mathrm{LP}(\mathcal{P}), \mathrm{Facts}(\mathcal{P}))$ be a ProbLog program, and let $|\mathcal{P}| = |\mathrm{LP}(\mathcal{P})| + |\mathrm{Facts}(\mathcal{P})|$ denote the total number of clauses and facts. Let $L_{\max}$ be the maximum body length of any clause in $\mathrm{LP}(\mathcal{P})$. Then the SWIP $\mathcal S(\mathcal{P})$ following Algorithm \ref{alg:swip_transform} has complexity $O(|\mathcal{P}| \cdot L_{\max})$.
\label{thrm:swip_complexity}   
\end{theorem}

As a corollary, it immediately follows
\begin{corollary}[Asymptotic Advantage of SWIFT over Twin Networks]
Let $\mathcal{P} = (\mathrm{LP}(\mathcal{P}), \mathrm{Facts}(\mathcal{P}))$ be a ProbLog program of size $|\mathcal{P}|$, and let $L_{\max}$ be the maximum body length of any clause in $\mathrm{LP}(\mathcal{P})$. 
Let $\mathcal{S}(\mathcal{P})$ be the Single-World Intervention Program produced by the SWIFT transformation, and let $\mathcal{T}(\mathcal{P})$ be the Twin Network transformation of $\mathcal{P}$.

Then the time complexity of constructing $\mathcal{S}(\mathcal{P})$ is $O(|\mathcal{P}| \cdot L_{\max})$, while the time complexity of constructing $\mathcal{T}(\mathcal{P})$ is $\Theta(|\mathcal{P}| \cdot L_{\max})$.
Moreover, $\mathcal{S}(\mathcal{P})$ is never asymptotically larger than $\mathcal{T}(\mathcal{P})$, and is strictly smaller whenever at least one intervened atom does not appear in all clause bodies.
\end{corollary} 

\begin{algorithm}[H]
\caption{\textsc{SWIFT}$(\mathcal{P},fix(\textbf{X}:=\textbf{x}))$}
\label{alg:swip_transform}
\begin{algorithmic}[1]
\Require ProbLog program $\mathcal{P} =(\mathrm{LP}(\mathcal{P}),\text{Facts}(\mathcal{P}))$; intervention $\textbf{X}:=\textbf{x}$
\Ensure Single-World Intervention Program $\mathcal{S}(\mathcal{P})$
 
\State Initialize $\mathcal{S}(\mathcal{P})\leftarrow \text{Facts}(\mathcal{P})$
\State Let $\mathrm{LP}_{-\textbf{X}}\leftarrow {C\in \mathrm{LP}(\mathcal{P}):head(C)\not\in \textbf{X}}$
\ForAll{$C=(h\leftarrow b_1,\dots ,b_n)\in LP_{-\textbf{X}}$}
    \State Let $B' ={b_1' ,\dots ,b_n'}$ be a new set of body literals
    \ForAll{$b_j \in {b_1 ,\dots ,b_n}$}
        \State Let $a$ be the atom of the literal $b_j$.
        \If{$a\in \textbf{X}$}
            \State $b_j' \leftarrow \text{literal corresponding to } a_{fixed}(x_a) \text{with the same sign as } b_j$.
        \Else
            \State $b_j' \leftarrow b_j$
        \EndIf
    \EndFor
    \State Add the rewritten rule $h\leftarrow B'$ to $\mathcal{S}(\mathcal{P})$
\EndFor
\ForAll{$X_i\in \textbf{X}$}
    \State Add the fact $1.0::X_{i,fixed}(x_i)$ to $\mathcal{S}(\mathcal{P})$
\EndFor
\State \Return $\mathcal{S}(\mathcal{P})$
\end{algorithmic}
\end{algorithm}

That is to say, while asymptotically equivalent in the worst case, SWIPs avoid unconditional duplication and are strictly smaller for sparse interventions. This directly translates to query speed ups in the general case, and in the worst case, is the same cost as querying over $\mathcal{T}(\mathcal{P})$.  As shown in Algorithm \ref{alg:swip_query}, evidence is incorporated by adding facts to $\mathcal{S}(\mathcal{P})$, and the counterfactual probability is obtained via standard marginal inference on this final program.

\begin{algorithm}[H]
\caption{\textsc{EvaluateSWIPQuery}$(\mathcal{S}(\mathcal{P}) ,\textbf{E} = \textbf{e},\phi )$}
\label{alg:swip_query}
\begin{algorithmic}[1]
\Require SWIP $\mathcal{S}(\mathcal{P})$; evidence $\textbf{E} = \textbf{e}$; counterfactual query formula $\phi $
\Ensure Counterfactual probability $\pi_{\mathcal{S}(\mathcal{P})}(\phi |\textbf{E} = \textbf{e})$
\State $p_1\leftarrow \pi_{\mathcal{S}(\mathcal{P})}(\textbf{E} = \textbf{e})$
\State $p_2\leftarrow \pi_{\mathcal{S}(\mathcal{P})}(\phi \wedge\textbf{E} = \textbf{e})$
\State \Return $p_2/p_1$
% \State Let $\mathcal{S}^e(\mathcal{P}) \leftarrow \mathcal{S}(\mathcal{P})\cup \{\text{evidence clauses for }\textbf{E}= \textbf{e}\}$
% \State Compute and return $\pi_{\mathcal{S}^e(\mathcal{P})}(\phi )$
\end{algorithmic}
\end{algorithm}

\begin{theorem}[Inference Complexity Comparison]
Let $g(\cdot)$ be the complexity of some inference algorithm. Then inference complexity for querying over $\mathcal S(\mathcal{P})$ vs $\mathcal T(\mathcal{P})$ is
$$O(g(w(\mathcal S(\mathcal{P})))) \leq O(g(w(\mathcal T(\mathcal{P}))))$$
\label{thrm:inference_complexity}
\end{theorem}

Because our transformation is proven to yield a program whose semantics are equivalent to the counterfactual distribution of the underlying SCM, it inherits the established consistency with other causal formalisms like CP-logic and LPADs \cite{Kiesel.etal2023,vennekens_cp-logic_2009,vennekens_logic_2004}.

\begin{theorem}[Correctness of SWIP-Based Counterfactual Queries]
Let $P$ be a ProbLog program encoding a structural causal model $M$ with unique supported models. Let $\mathcal S(\mathcal{P}) = \mathrm{SWIFT}(\mathcal{P},\, fix(\bf X:= \bf x))$ be the SWIG-transformed program and let $\mathcal S^e (\mathcal{P})$ be the program augmented with evidence $\bf E= \bf e$. Then for any query $\phi$ under intervention $\bf x$, 
\begin{equation*}
    \pi_{\mathcal S^e (\mathcal{P})}(\phi) = \pi_{\mathcal{M}}(\phi_{\textbf{x}} | \textbf{E} = \textbf{e}).
\end{equation*}
\label{thrm:swip_correctness}
\end{theorem}

These theorems suggest with respect to program construction, the Twin Network approach unconditionally duplicates the program, regardless of the intervention query while the SWIP approach scales slower with program size and maximum clause body length and at worst case, will incur the same cost as the Twin Network approach. In realistic applications, where programs and evidence sets may be large but interventions focused on a small subset of predicates, SWIPs can take advantage of the locality and specificity of the query structure. As a result, SWIPs naturally favor compact counterfactual representations and furthermore, are a more realistic approach for capturing real world counterfactuals. Beyond the issue of cross-world independence assumptions growing at a doubly exponential rate, \cite{richardson_single_2013} strongly emphasize the assumptions of the Twin Network are, by definition, mutually exclusive and thus, experimentally unverifiable.  SWIGs, and by extension, SWIPs, ability to avoid this is especially salient in  sequentially randomized trials and longitudinal decision problems. In such settings, treatments are assigned over time based on evolving histories, and counterfactual reasoning must respect the temporal and logical structure of these assignments. Richardson and Robins show that Twin Network constructions can induce incorrect independencies in these cases, whereas SWIGs preserve the correct causal structure by explicitly representing interventions as node-splitting operations within a single world \cite{richardson_single_2013}. Our results show that SWIPs inherit this advantage at the level of probabilistic logic programs: interventions modify only the clauses corresponding to treatment assignment mechanisms, while leaving downstream deterministic and probabilistic dependencies intact.

We can characterize these counterfactual independence conclusions in PLP contexts when not using SWIPs and show our method is a more general and robust implementation of the SCM semantics and is consistent with CP-logic over the same set of models as the Twin Network. 
\begin{theorem}[SWIP Consistency with LPAD ]
Let $\mathcal{P}$ be a propositional LPAD-program such that every selection yields a logic program with a unique supported model. Let $\mathbf{X}, \mathbf{E} \subseteq \mathfrak{B}$ be sets of propositions with value assignments $\mathbf{x}$ and $\mathbf{e}$, respectively, and let $\phi$ be a $\mathfrak{P}$-formula.  

Denote by $\pi^{CP}_{\mathcal{P}}(\phi \mid \textbf{E}=\textbf{e}, fix(\textbf{X}:=\textbf{x}))$ the counterfactual probability computed by CP-logic using the fixed-operator semantics on LPADs, and by
\[
\pi^{SWIP}_{\mathrm{Prob}(\mathcal{P})}(\phi \mid \textbf{E}=\textbf{e}, fix(\textbf{X}:=\textbf{x}))
\]
the counterfactual probability induced by the corresponding Single World Intervention Program (SWIP) constructed from $\mathrm{Prob}(\mathcal{P})$.

Then,
\[
\pi^{CP}_{\mathcal{P}}(\phi \mid \textbf{E}=\textbf{e}, fix(\textbf{X}:=\textbf{x}))
\;=\;
\pi^{SWIP}_{\mathrm{Prob}(\mathcal{P})}(\phi \mid \textbf{E}=\textbf{e}, fix(\textbf{X}:=\textbf{x})).
\]
\label{thrm:swip_cp1}
\end{theorem}

\begin{theorem}[Consistency of SWIPs with CP-Logic]
Let $\mathcal{P}$ be a ProbLog program with unique supported models, and let $\mathbf{X}, \mathbf{E} \subseteq \mathfrak{B}$ with value assignments $\mathbf{x}$ and $\mathbf{e}$. For any $\mathfrak{P}$-formula $\phi$, denote by
\[
\pi^{CP}_{\mathrm{LPAD}(\mathcal{P})}(\phi \mid \textbf{E}=\textbf{e}, fix(\textbf{X}:=\textbf{x}))
\]
the counterfactual probability defined via CP-logic on the LPAD-transformation of $\mathcal{P}$, and by
\[
\pi^{SWIP}_{\mathcal{P}}(\phi \mid \textbf{E}=\textbf{e}, fix(\textbf{X}:=\textbf{x}))
\]
the counterfactual probability computed by the SWIP semantics directly on $\mathcal{P}$.

Then,
\[
\pi^{CP}_{\mathrm{LPAD}(\mathcal{P})}(\phi \mid \textbf{E}=\textbf{e}, fix(\textbf{X}:=\textbf{x}))
\;=\;
\pi^{SWIP}_{\mathcal{P}}(\phi \mid \textbf{E}=\textbf{e}, fix(\textbf{X}:=\textbf{x})).
\]
\label{thrm:swip_cp2}
\end{theorem}

For example, consider a simple power failure system that involves a deterministic relationship that creates a logical constraint.  Let $A$ be the main power supply, $B$ the independent backup power supply, $C$ an indicator for if the system is on, and $D$ a deterministic report filed iff $C$ is true. As a ProbLog program, $\mathcal{P}$: 
\begin{verbatim}
% Exogenous variables for the two power supplies
0.5::u_a
0.5::u_b

% Endogenous variables defined by structural rules
a :- u_a.
b :- u_b.
c :- a.
c :- b.
d :- c.
\end{verbatim}

The atom $a$ represents the main power being active, $b$ represents the backup being active, $c$ represents the system being on, and $d$ represents the report being filed. The rule $\texttt{d :- c}$ establishes the deterministic link. Now consider the query, "Given we know the back up generator was on, is the report being filed independent of the main generator?" or formally, $D\bot A| B=1$. According to standard do-calculus on the Twin Network of this program, we see that $D$ is not independent of $A$ but of course, the status of $A$ is no longer relevant as $B=1$ fully informs us of $D$. Indeed, in the ProbLog d-separation metainterpreter by \cite{ruckschlos_subtlety_nodate}, the twin network program will not reveal this dependency but it is trivially identified in the SWIP. 

The SWIG literature emphasizes that single-world node-splitting correctly exposes many counterfactual conditional independencies that cannot be read from a naive multi-world duplication via \textit{d}-separation without further assumptions \cite{richardson_single_2013}. All \textit{d}-separation claims in this paper are evaluated on the grounded dependency graph corresponding to the induced Structural Causal Model of the ProbLog program. Concretely, this graph is obtained by grounding the program and interpreting probabilistic facts as mutually independent exogenous variables and logical clauses as directed functional dependencies between endogenous variables. The following formal statement makes this concrete in the ProbLog setting and generalizes the power failure example.

\begin{theorem}[Misidentified Independencies in Twin Network Programs]
\label{thm:twin-fail}
Let $\mathcal P$ be a ProbLog program encoding an SCM $\mathcal{M} = \langle \mathbf U, \mathbf V, \mathcal{F} \rangle$. Suppose there exist distinct endogenous atoms $A,B,D\in\mathbf V$ such that:

\begin{enumerate}[leftmargin=2em, label=(\arabic*)]
    \item $D$ is defined in $\mathrm{LP}(\mathcal P)$ by a set of clauses whose combined effect is a deterministic function $d = g(a,b,\mathbf{u}_D)$ (possibly expressed via multiple rules), where $\mathbf{u}_D$ denotes the exogenous input(s) relevant to $D$;
    \item there exists a value $b^\star$ for $B$ with the \emph{screening property}
    $$
      g(a,b^\star,u_D) = g(a',b^\star,u_D) \quad\ \forall a,a' \text{ and all } \mathbf{u}_D,
    $$
    i.e. when $B=b^\star$ the value of $D$ is pointwise independent of $A$ given the same exogenous input $u_D$;
    \item $A$ and $B$ are not d-separated by the empty set in the causal graph underlying $\mathcal P$ and the program satisfies unique supported model conditions so the ProbLog semantics is well defined.
\end{enumerate}

Then the following hold:
\begin{enumerate}[leftmargin=2em, label=(\arabic*)]
  \item[(a)] On the SWIP for the intervention $\do(B{=}b^\star)$ the node $D_{b^\star}$ is d-separated from $A$ (possibly conditioning on nothing or on appropriate observed variables), and hence the SWIP implies the counterfactual independence $D \perp A|B=b^* $
  \item[(b)] There exist programs $\mathcal P$ satisfying (1) - (3) for which the Twin Network construction $\mathcal T(\mathcal P)$ does not d-separate the factual atom $A$ from the counterfactual copy $D^\ast$. Consequently, ordinary \textit{do}-calculus on $\mathcal T(\mathcal P)$ will not soundly identify the independence. 
\end{enumerate}
\end{theorem}

In the power failure example, it is trivial to see this is a result of the direct dependency between $D$ and $C$. However, the value of a counterfactual program transformation which is a complete independence oracle is highlighted by condition (1) of Theorem \ref{thm:twin-fail}, where $D$ may be expressed, unobviously, deterministically as a composition of multiple rules. The above results demonstrate the failure of the Twin Network is not a PLP issue nor can PLP alone reveal independencies in program structure when interventions are implemented across worlds \cite{ruckschlos_subtlety_nodate, R_ckschlo__2023}. SWIPs, by design, align functional causal semantics with the interventional theory of CP-Logic while avoiding independence pathologies. 

% See Appendix \ref{app:dsep} for more details.

% \begin{algorithm}[H]
% \caption{\textsc{SWIFT}$(\mathcal{P},fix(\textbf{X}:=\textbf{x}))$}
% \label{alg:swip_transform}
% \begin{algorithmic}[1]
% \Require ProbLog program $\mathcal{P} =(\mathrm{LP}(\mathcal{P}),\text{Facts}(\mathcal{P}))$; intervention $\textbf{X}:=\textbf{x}$
% \Ensure Single-World Intervention Program $\mathcal{S}(\mathcal{P})$
 
% \State Initialize $\mathcal{S}(\mathcal{P})\leftarrow \text{Facts}(\mathcal{P})$
% \State Let $\mathrm{LP}_{-\textbf{X}}\leftarrow {C\in \mathrm{LP}(\mathcal{P}):head(C)\not\in \textbf{X}}$
% \ForAll{$C=(h\leftarrow b_1,\dots ,b_n)\in LP_{-\textbf{X}}$}
%     \State Let $B' ={b_1' ,\dots ,b_n'}$ be a new set of body literals
%     \ForAll{$b_j \in {b_1 ,\dots ,b_n}$}
%         \State Let $a$ be the atom of the literal $b_j$.
%         \If{$a\in \textbf{X}$}
%             \State $b_j' \leftarrow \text{literal corresponding to } a_{fixed}(x_a) \text{with the same sign as } b_j$.
%         \Else
%             \State $b_j' \leftarrow b_j$
%         \EndIf
%     \EndFor
%     \State Add the rewritten rule $h\leftarrow B'$ to $\mathcal{S}(\mathcal{P})$
% \EndFor
% \ForAll{$X_i\in \textbf{X}$}
%     \State Add the fact $1.0::X_{i,fixed}(x_i)$ to $\mathcal{S}(\mathcal{P})$
% \EndFor
% \State \Return $\mathcal{S}(\mathcal{P})$
% \end{algorithmic}
% \end{algorithm}

\section{Experiments}
\label{sec:exp}

We've established that SWIPs can be used for counterfactual queries by reducing them to marginal inference and carry out this query via SharpSAT, a top down Knowledge Compilation \cite{korhonen_integrating_2021}. To assess the scalability and efficiency of our SWIP counterfactual inference method, we replicate and extend the experimental setup from \cite{kiesel_what_2023}. 
Our evaluation focuses on three primary questions: (i) how counterfactual program size varies with the SWIP approach and the Twin Network approach, (ii) how inference time varies with program size and structural complexity, and (iii) how inference time is affected by the number and type of evidence and intervention atoms.  

\myparagraph{Benchmark Instances}
Following \cite{kiesel_what_2023}, we generate acyclic directed graphs (DAGs) with a controlled size and treewidth. Each instance corresponds to a random probabilistic logic program modeling reachability in a directed graph. 
For a given graph $G=(V,E)$ with distinguished start and goal nodes $s,g\in V$, we encode the probability of reaching $g$ from $s$ using the following ProbLog schema:
\\
\begin{verbatim}
r(s).
0.1::trap(Y) :- p(X,Y).
r(Y) :- p(X,Y).
1/d(X)::p(X, s 1(X));...;1/d(X)::p(X, s d(X)):- r(X), \+ trap(X).
\end{verbatim}
Here, $d(X)$ denotes the out-degree of vertex $X$, and $\texttt{s\_i}(X)$ denotes its $i$-th child node. 
The resulting program represents the random process of traversing the graph from $s$ to $g$, avoiding nodes marked as traps.  We vary two parameters controlling instance difficulty: the number of vertices $n$ and the treewidth $k$. 
We first generate a random tree of size $n$ using the \texttt{networkx} library (which has treewidth $1$), and then add $k$ additional nodes, each connected by incoming arcs from randomly selected original nodes. 
Finally, a single goal vertex $g$ is added, receiving edges from each of the $k$ new nodes. 
This procedure yields a DAG of size $n+k+1$ with treewidth $\min(n,k)$ and ensures acyclicity. 
For each instance, we sample up to five pieces of evidence and five interventions, allowing us to examine the interaction between query complexity and inference performance.
Counterfactual queries are defined over this model by introducing positive or negative evidence on intermediate reachability predicates and positive or negative interventions on selected edges, following the schema:
$
\pi^{\mathcal{M}}_{\mathcal{P}}\big(r(g)\,|\,\neg r(v_1), \ldots, \neg r(v_n), do(\neg r(v'_1)), \ldots, do(\neg r(v'_m))\big),
$
for some evidence nodes $v_1,\ldots,v_n$ and intervened nodes $v'_1,\ldots,v'_m$. 

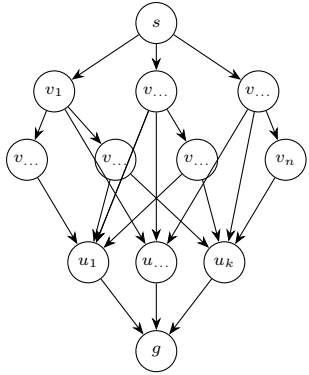
\begin{figure}[ht]
\centering
\begin{tikzpicture}[
  scale=0.9,
  every node/.style={transform shape},
  node/.style={circle,draw,minimum size=6mm,inner sep=0pt,font=\scriptsize},
  tree/.style={node},
  added/.style={node},
  start/.style={node},
  goal/.style={node},
  edge/.style={->, >=Stealth, thin},
  guide/.style={- , gray, thin}
]

% --- Top: Random tree ---
\node[start] (s) at (0,1.6) {$s$};
\node[tree] (t1) at (-1.5,0.6) {$v_1$};
\node[tree] (t2) at (0,0.6) {$v_{\dots}$};
\node[tree] (t3) at (1.5,0.6) {$v_{\dots}$};
\node[tree] (t4) at (-1.9,-0.4) {$v_{\dots}$};
\node[tree] (t5) at (-0.6,-0.4) {$v_{\dots}$};
\node[tree] (t6) at (0.6,-0.4) {$v_{\dots}$};
\node[tree] (t7) at (1.9,-0.4) {$v_n$};

\draw[edge] (s) -- (t1);
\draw[edge] (s) -- (t2);
\draw[edge] (s) -- (t3);
\draw[edge] (t1) -- (t4);
\draw[edge] (t1) -- (t5);
\draw[edge] (t2) -- (t6);
\draw[edge] (t3) -- (t7);

% --- Middle: Dense layer (below tree) ---
\node[added] (a1) at (-1.0,-1.9) {$u_1$};
\node[added] (a2) at (0.0,-1.9)  {$u_{\dots}$};
\node[added] (a3) at (1.0,-1.9)  {$u_k$};

% edges from tree nodes to dense-layer nodes
\foreach \x in {t4,t5,t6} \draw[edge] (\x) -- (a1);
\foreach \x in {t1,t2,t3} \draw[edge] (\x) -- (a2);
\foreach \x in {t5,t6,t7} \draw[edge] (\x) -- (a3);
\draw[edge] (t2) -- (a1);
\draw[edge] (t2) -- (a1);
\draw[edge] (t3) -- (a3);

% --- Bottom: Goal node (below dense layer) ---
\node[goal] (g) at (0,-3.2) {$g$};
\draw[edge] (a1) -- (g);
\draw[edge] (a2) -- (g);
\draw[edge] (a3) -- (g);

% % --- Right-side layer labels with small guide lines ---
% \node[align=left,font=\small] (lbl_tree) at (-4,1.2) {\textbf{Random tree}\$size $n$, root $s$)};
% \draw[guide] (t2.east) .. controls +(0.6,0.1) and +(-0.3,0.4) .. (lbl_tree.west);

% \node[align=left,font=\small] (lbl_dense) at (-4,-0.6) {\textbf{Dense layer}\$$k$ nodes $u_1,\dots,u_k$)};
% \draw[guide] (a2.east) .. controls +(0.6,0) and +(-0.3,-0.1) .. (lbl_dense.west);

% \node[align=left,font=\small] (lbl_goal) at (2.6,-3.2) {\textbf{Goal node}\\$g$};
% \draw[guide] (g.east) .. controls +(0.6,0) and +(-0.1,0) .. (lbl_goal.west);z

\end{tikzpicture}
\caption{Construction of benchmark DAGs. A random tree of size $n$ rooted at $s$ (top) is generated, a dense layer of $k$ nodes $u_1,\dots,u_k$ is added below (each receiving edges from multiple tree nodes), and a single goal node $g$ is appended below the dense layer receiving edges from every $u_i$.}
\label{fig:dag-construction}
\end{figure}

All experiments were executed on the University of Edinburgh’s compute cluster. Each node is equipped with two Intel Xeon Gold 5218 CPUs (16 cores per CPU, 2.30~GHz base frequency), with 256~GB of DDR4 RAM per node operating at 2666~MHz. 
The cluster runs Red Hat Enterprise Linux~8.6 and uses the Slurm workload manager for parallel execution. 
We performed all experiments using Python~3.9.21 and ProbLog~2.2, with inference powered by the \texttt{SHARPSAT} knowledge compiler for top-down inference. This compiler was shown in \cite{kiesel_what_2023} to be fastest for counterfactual program query compilation.
Each query was given a time limit of 1800~seconds; queries reaching the limit were assigned this maximum runtime for consistency with \cite{kiesel_what_2023}. 

\subsection{Results and Discussion}
\label{subsec:results}

Figure~\ref{fig:tree_lines} shows the unfolded treewidths of the compiled ProbLog programs as a function of the synthetic graph size and original treewidth. 
As predicted, the SWIG-based transformation produces programs with substantially smaller unfolded treewidths than the Twin Network baseline. 
This difference arises because our approach performs a localized “graph surgery,” modifying only the clauses corresponding to intervened variables, rather than duplicating the entire logical program. 
By maintaining a single-world representation, the resulting dependency graph remains tighter and less entangled, leading to more compact knowledge-compilation structures. 
These smaller unfolded programs directly reduce the complexity of subsequent inference procedures, since treewidth is the dominant factor determining compilation cost in ProbLog-based inference \cite{Kiesel.etal2023, eiter_treewidth-aware_2021, korhonen_integrating_2021}.

\begin{figure}[h]
  \centering
  \includegraphics[width=0.85\textwidth]{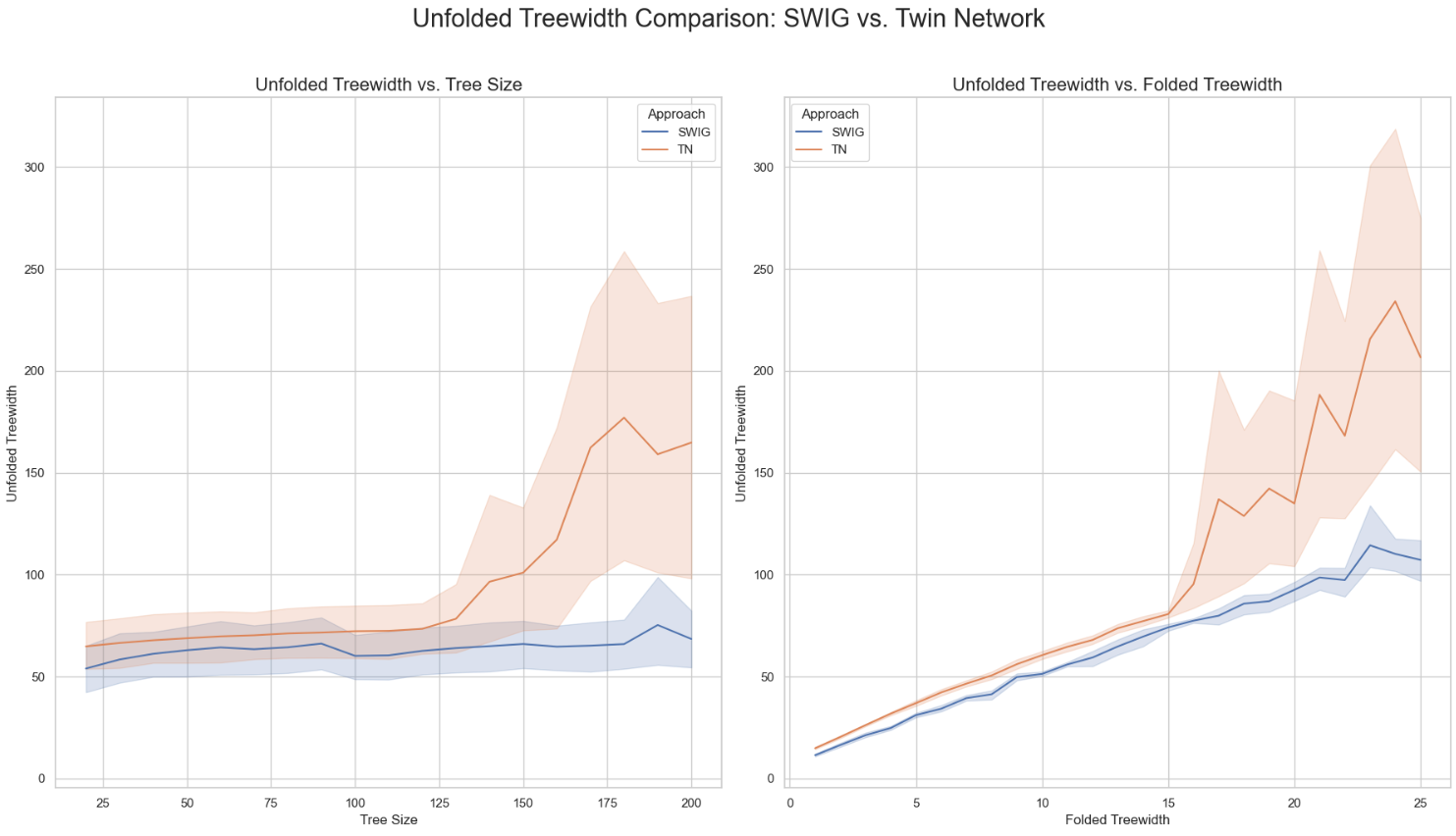}
  \caption{Unfolded treewidth as a function of synthetic graph size and original treewidth. 
  The SWIG-based transformation yields narrower unfolded dependency structures than the Twin Network approach.}
  \label{fig:tree_lines}
\end{figure}

This structural advantage translates directly into computational gains, as illustrated in Figure~\ref{fig:execution_times}. 
Across all benchmark instances, the SWIG-based programs consistently compile and evaluate faster than their Twin Network counterparts. 
On average, our method requires approximately 65\% of the runtime of the baseline for equivalent evidence and intervention configurations. 
The performance gap remains stable across varying program sizes and query complexities, confirming that reducing unfolded treewidth yields measurable improvements in both compilation and inference phases. 
Together, these results empirically validate that the SWIG-based transformation preserves the expressive and causal semantics of counterfactual reasoning while significantly improving computational efficiency.

\begin{figure}[h]
  \centering
  \includegraphics[width=0.85\textwidth]{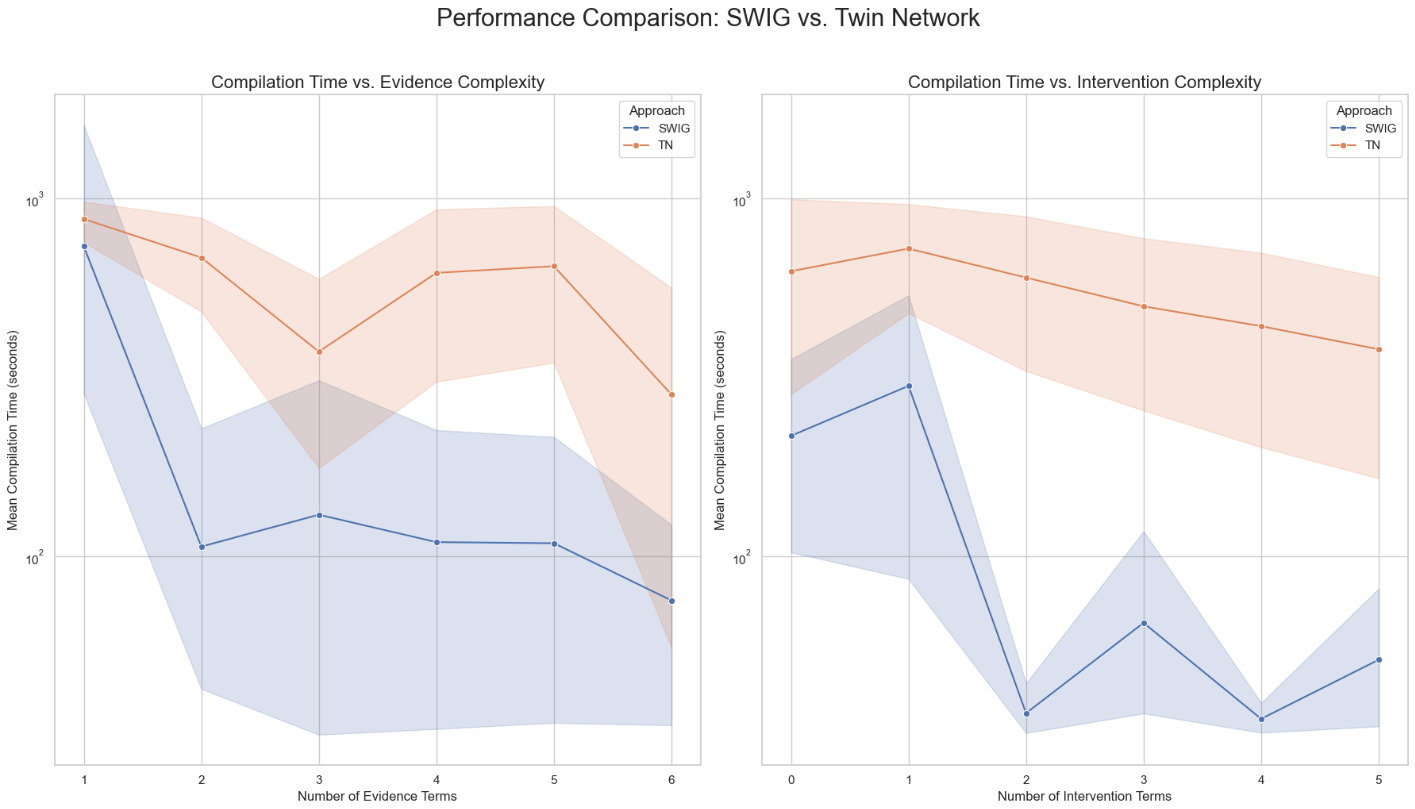}
  \caption{Mean compilation and inference times for varying numbers of evidence and intervention atoms. 
  The SWIG-based transformation achieves lower inference and compilation times, corresponding to its smaller unfolded program structures.}
  \label{fig:execution_times}
\end{figure}

\section{Conclusion}
\label{sec:conclusion}
We have presented a SWIG-based approach to counterfactual reasoning in ProbLog. By “fixing” intervened nodes and propagating this intervention to descendant rules, we are able to reduce counterfactual queries to standard marginal inference queries. Our procedure of a transformation under intervention  and subsequent evidence incorporation, is  proven to yield the counterfactual distribution $P(\cdot\mid do(X=x),\, \textbf{E} = \textbf{e})$ under assumptions of consistency and modularity. It also provides significant computational speedups in inference compared to existing approaches for counterfactual reasoning in ProbLog. In future work, it would be interesting to apply this approach to real-world settings and establishing counterfactual reasoning capabilities to DeepProbLog, an extension of ProbLog that combines it with neural network predicates to combine high level logical reason with low level subsymbolic perception \cite{manhaeve_neural_2021}. Furthermore, exploring $\sigma$-calculus, a more general form of \textit{do}-calculus \cite{correa_calculus_2020, forre_constraint-based_2018}, to define counterfactuals for programs without unique supported models offers promising research directions. 

\section*{Acknowledgments}
Funding for this research was provided by NERC through an E4 DTP studentship (NE/S007407/1).

\bibliographystyle{ieeetr}
\bibliography{references}

% \begin{thebibliography}{}
%   \bibitem[\protect\citename{Akmajian and Lehrer, }1976]{akm76}
%    Akmajian \& Lehrer A. 1976. NP  like quantifiers and the
%    problem of determining the head of an NP. {\it Linguistic
%    Analysis\/} {\it 11}, 1, 295    313.
%   \bibitem[\protect\citename{Huddleston, }1984]{hud84}
%    Huddleston, Rodney. 1984. {\it Introduction to the Grammar of
%    English}. Cambridge: Cambridge University Press.
%   \bibitem[\protect\citename{McCord, }1990]{mcc90}
%    McCord, Michael C. 1990. Slot grammar: a system for simpler
%    construction of practical natural language grammars. In R.
%    Studer (ed.), {\it Natural Language and Logic: International
%    Scientific Symposium}, pp.~118    45. Lecture Notes in Computer
%    Science. Berlin: Springer  Verlag.
%   \bibitem[\protect\citename{Salton {\it et al.}, }1990]{sal90}
%    Salton, Gerald, Zhao, Zhongnan \& Buckley, Chris. 1990.
%    A simple syntactic approach for the generation of indexing
%    phrases. Technical Report 90    1137, Department of Computer
%    Science, Cornell University.
% \end{thebibliography}
\newpage
\appendix
\section{Notation and Definitions}
\begin{table}[h!]
\caption{Summary of key notation used throughout the paper.}
\label{tab:notation}
\centering
\small
\begin{tabular}{ll}
\hline
\textbf{Symbol} & \textbf{Definition} \\
\hline

$\mathbf{V}$ &
Set of endogenous variables \\

$\mathbf{U}$ &
Set of exogenous variables \\

$\mathcal{F}$ &
Set of functional equations \\

$pa(X)$ &
Set of endogenous parents of $x$ \\

$\mathcal{M}$ &
Structural causal model (SCM): tuple of $\langle \textbf{U}, \textbf{V}, \mathcal{F}\rangle$ \\

$X := f_X(\mathrm{pa}(X), \epsilon_X)$ &
Structural equation for variable $X$ w\\

% $\mathbf{X}, \mathbf{x}$ &
% Intervened variables and their assigned values \\
$\mathrm{do}(\mathbf{X} := \mathbf{x})$ &
Pearl’s do-operator: modifies SCM equations of involving $\mathbf X$ \\

$\mathrm{fix}(\mathbf{X} := \mathbf{x})$ &
Fix-operator: constructs a Single-World Intervention Graph\\

$\mathbf{V}^*$ &
Counterfactual copies of variables in the Twin Network \\

$\mathcal{M}^K$ &
Twin Network SCM with shared exogenous variables \\

% $Y(\mathbf{x})$ &
% Potential outcome under treatment $\mathbf{X} = \mathbf{x}$ \\

$\mathcal{P}$ &
ProbLog program with logic and random facts \\

$\mathrm{LP}(\mathcal{P})$ &
Underlying logical clauses in a ProbLog program \\

$\mathfrak{B}$ &
Propositional alphabet for logic atoms \\

$\mathcal{T}(\mathcal{P})$ &
Transformed Twin Network ProbLog program \\

$\mathcal{S}(\mathcal{P})$ &
Transformed Single World Intervention Program \\

\hline
\end{tabular}
\end{table}

\section{Proofs of Theorems \ref{thrm:tn_complexity}, \ref{thrm:swip_complexity}, \ref{thrm:inference_complexity}}

%\myparagraph{Theorem \ref{thrm:tn_complexity}}
\begin{proof}[Proof of Theorem \ref{thrm:tn_complexity}]
    Consider that for each fact $r :: A \in \mathrm{Facts}(\mathcal{P})$, we either intervene on $r$ if $A$ is in the intervention set $\textbf{X}$ or we add the observed and counterfactual facts to our program $\mathcal{T}(\mathcal{P})$. This has a complexity of $\Theta(|\mathrm{Facts}(\mathcal{P})|)$. 
    
    Similarly, for each clause $C = h \leftarrow b_1, \ldots, b_n \in \mathrm{LP}(P)$ with $h \notin \textbf{X}$:, we add the clause and the counterfactual copy to $\mathcal{T}(\mathcal{P})$. Contructing the body of the counterfactual clause takes $\Theta(n)$ per clause. Thus the total rule transformation complexity is:
    \[
        \sum_{C \in \mathrm{LP}(P)} \Theta(|\mathrm{body}(C)|) = \Theta(|\mathrm{LP}(P)| \cdot L_{\max})
    \]
Summing these steps together we get a complexity of $\Theta(|\mathcal{P}|\cdot L_{\max})$.
\end{proof}

%\myparagraph{Theorem \ref{thrm:swip_complexity}}
\begin{proof}[Proof of Theorem \ref{thrm:swip_complexity}]
  Analogous to Theorem \ref{thrm:tn_complexity}, except we do not need to rewrite every clause, just the clauses with variables in the intervention set. Thus, we have $O(|\mathcal{P}|\cdot L_{\max})$
\end{proof}

%\myparagraph{Theorem \ref{thrm:inference_complexity}}
\begin{proof}[Proof of Theorem \ref{thrm:inference_complexity}]
Let $A$ be the set of internal atoms. Then $\mathcal{T}(\mathcal{P})$ contains approximately $2|A|$ atoms and rules. However, as the primal graph of $\mathcal{T}(\mathcal{P})$, denoted $G_{\mathcal{T}(\mathcal{P})}$, is constructed over the set of endogenous atoms $A \cup A^*$ and no rule in $\mathcal{T}(\mathcal{P})$ contains both an atom from $A$ and an atom from $A^*$ in its body, the resulting primal graph consists of two disconnected components. The treewidth of this composite graph, $w(\mathcal{T}(\mathcal{P}))$, is therefore equal to the treewidth of the primal graph of the original program, $w(\mathcal{P})$.
Since inference (e.g., via knowledge compilation) has complexity as a function of treewidth, we have
\[
 O(g(w(\mathcal{T}(\mathcal{P})))) = O(g(w(\mathcal{P})))
\]
Let $G_{\mathcal S (\mathcal{P})}$ be the primal graph of the transformed program $\mathcal S (\mathcal{P})$. Note that removing rules with heads in $\textbf{X}$ deletes edges in $G_P$ and that rewriting variables with fixed values does not add new dependencies. Therefore:
    \[
    w(\mathcal{S}(\mathcal{P})) \leq w(P)
    \]
Hence, inference complexity is:
\[
    O(g(w(\mathcal{S}(\mathcal{P})))) \leq O(g(w(\mathcal{T}(\mathcal{P}))))
    \]
\end{proof}

\section{Proofs of Theorem \ref{thrm:swip_cp1} and Theorem \ref{thrm:swip_cp2}}

To demonstrate the consistency of the SWIP treatment of counterfactuals with CP-logic we start by recalling the theory of CP-Logic from Vennekens et al. (2009) and the LPAD-programs of Vennekens et al. (2004) with their standard semantics. An LPAD program, $\bf P$, is a finite set of rules of the following form:

\[
RC := h_1 : \pi_1 ; \ldots ; h_l : \pi_l \leftarrow b_1, \ldots, b_n
\]

where $h_i$ and $b_i$ are atoms and literals in some set of propositions $\mathfrak{B}$. We have that the $\pi_i \in [0,1]$ are associated probabilities for $h_i$ such that $\sum_i \pi_i \leq 1$. We define $\text{head}(RC) := (h_1, \ldots, h_l)$ to be a tuple of propositions which are the head of $RC$, where $h \in (h_1, \ldots, h_l)$ if $h = h_i$ for a $1 \leq i \leq l$ and furthermore, $l(RC) := l$ and $h_i(RC) := h_i$ for $1 \leq i \leq l$. Similarly, the body of $RC$ is the finite set of literals $\text{body}(RC) := \{b_1, \ldots, b_n\}$

A selection, $\sigma$, of $\bf{P}$ is a function $\sigma : \mathcal{P} \to \mathbb{N} \cup \{\bot\}$, where $\bot \notin \mathbb{N}$, that assigns to each LPAD-clause $RC \in \bf{P}$ a natural number in $[1,l]$ or $\sigma(RC) := \bot$. To each selection $\sigma$, we associate a probability
\[
\pi(\sigma) := \prod_{\substack{RC \in \mathcal{P} \\ \sigma(RC) \in \mathbb{N}}} \pi_{\sigma(RC)}(RC) 
\prod_{\substack{RC \in \mathcal{P} \\ \sigma(RC)=\bot}} 
\left( 1 - \sum_{i=1}^{l(RC)} \pi_i(RC) \right)
\]
and each selection gives a logic program
\[
\mathcal{P} ^\sigma := \{ h_{\sigma(RC)} \leftarrow \text{body}(RC) : RC \in \mathcal{P}, \sigma(RC) \neq \bot \}.
\]
We define $\pi^{dist}$,  the distribution semantics of $\mathcal{P}$, as 
\[
\pi^{dist}_{\mathcal{P}}(\phi) := \sum_{\substack{\sigma \text{ selection} \\ \mathcal{P}^\sigma \models \phi}} \pi(\sigma).
\]
where $\phi$ is some $\mathfrak{P}$-formula. 

Rigguzi (2020) \S2.4 establishes we can can translate an LPAD program $\mathcal{P}$ in $\mathfrak{B}$ to a ProbLog program, $\textrm{Prob}(\mathcal{P})$ where the logic of the program $LP(\mathrm{Prob}(\mathcal{P})$ is given by choosing distinct propositions $h_i^{RC}, u_i(RC)\not\in\mathfrak{B}$ for all $RC \in \mathcal{P}$ and natural number in $[1, l]$ and letting
\[
h_i^{RC} \leftarrow \text{body}(RC) \cup \{\neg h_j^{RC} \mid 1 \leq j < i\} \cup \{u_i(RC)\},
\]
\[
h_i \leftarrow h_i^{RC}
\]
Meanwhile, the random facts, $\mathrm{Facts}(\mathrm{Prob}(\mathcal{P}))$ are defined as 
\[
\text{Facts}(\text{Prob}(\mathcal{P})) := \left\{ \frac{\pi_i(RC)}{1 - \prod_{1 \leq j < i} \pi_j(RC)} :: u_i(RC) \;\middle|\; RC \in \mathcal{P}, 1 \leq i \leq l \right\}.
\]

This then ensures the following:

\begin{theorem}[Riguzzi (2020), §2.4]
\label{thrm:riguzzi_1}
Let $\mathcal{P}$ be a LPAD-program. Then, for every selection $\sigma$ of $\mathcal{P}$ a set of possible worlds $\mathcal{E}(\sigma)$, which consists of all possible worlds $\mathcal E$ such that $\neg u_i(RC)$ holds unless $\sigma(RC) \neq \bot$ or $i > \sigma(RC)$ and such that $u_{\sigma(RC)}(RC)$ holds for every $RC \in \mathcal{P}$ with $\sigma(RC) \neq \bot$.  
We conclude that $\mathcal{P}^\sigma$ yields the same answer to every $\mathfrak{B}$-formula as the logic programs $\text{LP}(\text{Prob}(\mathcal{P})) \cup \xi$ for every $\xi \in \mathcal{E}(\sigma)$ and that $\pi(\mathcal{E}(\sigma)) = \pi(\sigma)$. Further, the distribution semantics $\pi^{dist}_{\mathcal{P}}$ of $\mathcal{P}$ and the distribution semantics $\pi^{dist}_{\text{Prob}(\mathcal{P})}$ of $\text{Prob}(\mathcal{P})$ yield the same joint distribution on $\mathcal{P}$. 
\end{theorem}

Conversely, we can turn each ProbLog program to an equivalent LPAD-Program. Again, Riguzzi (2020), \S2.4 establishes for a ProbLog program $\mathcal{P}$ the LPAD-transformation $\text{LPAD}(\mathcal{P})$ is the LPAD-program that consists of one clause of the form $u(RF) : \pi(RF) \leftarrow$ for every random fact $\pi(RF) :: u(RF)$ of $\mathcal{P}$ and a clause of the form $\text{head}(LC) : 1 \leftarrow \text{body}(LC)$ for every logic clause $LC \in \text{LP}(\mathcal{P})$.  
In this case, every selection $\sigma$ of $\text{LPAD}(\mathcal{P})$ of probability not zero corresponds to a unique possible world $\mathcal{E}(\sigma)$, in which $u(RC)$ is true if and only if $\sigma(RC) \neq \bot$.

Again, we obtain that the LPAD-transformation respects the distribution semantics.

\begin{theorem}[Riguzzi (2020), §2.4]
\label{thrm:riguzzi_2}
By the transformation of $\mathcal{P}$ to the LPAD program $\text{LPAD}(\mathcal{P})$, we have $\text{LP}(\mathcal{P}) \cup \mathcal{E}(\sigma)$ and $\text{LPAD}(\mathcal{P})^\sigma$ yield the same answer to every $\mathcal{P}$-formula. We also get that $\pi(\sigma) = \pi(\mathcal{E}(\sigma))$. Hence, $\mathcal{P}$ and $\text{LPAD}(\mathcal{P})$ yield the same probability for every $\mathcal{P}$-formula.
\end{theorem}

Up until now, we have not differed from Kiesel et al. (2023) proof for showing the equivalence of interventions and counterfactuals in CP-Logic with SCMs by the $do$-operator in Twin Networks. However, to demonstrate that the equivalence holds under the $fix$-operator, we must extend the above theorems with the lemma which follows (and is proved) closely in style to Kiesel et al. (2023) Appendix B. 
\begin{lemma}
\label{lem:fix_cp}
Choose a proposition $X \in \mathcal{P}$ together with a truth value $x$.
\begin{enumerate}
\item In the situation of Theorem \ref{thrm:riguzzi_1}, for every possible world $\mathcal{E} \in \mathcal{E}(\sigma)$ the logic programs $\mathcal{P}^\sigma_{fix(X:=x)}$ and $\text{LP}(\text{Prob}(\mathcal{P})_{fix(X:=x)}) \cup \mathcal{E}$ yield the same answer to every $\mathfrak{P}$-formula.
\item In the situation of Theorem \ref{thrm:riguzzi_2}, for every selection $\sigma$ of $\text{LPAD}(\mathcal{P})$, the logic programs $\text{LPAD}(\mathcal{P})^\sigma_{fix(X:=x)}$ and $\text{LP}(\text{Prob}(\mathcal{P})_{fix(X:=x)}) \cup \mathcal{E}(\sigma)$ yield the same answer to every $\mathfrak{P}$-formula.
\end{enumerate}
\end{lemma}

\begin{proof}
To prove (1), consider by Theorem B.1, for every $\mathfrak{P}$-formula $\phi$,  $\mathcal{P}^\sigma$ and $\text{LP}(\text{Prob}(\mathcal{P})) \cup \mathcal{E}$ will give the same answer to the query because in both programs are modular and consequently, their behavior is invariant to erasing clauses with $X$ in the head or adding $X \leftarrow$. More precisely, the SWIG style $fix(X:=x)$-operation in a LPAD/ProbLog syntactic setting is implemented by removing every head-option $h:\pi$ from every LPAD-clause whose head contains $X$, and adds some deterministic fact $X \leftarrow 0/1$. 

This implements the same surgical operations as used in implementing the $do$-operator in earlier proofs. If we fix a selection $\sigma$ of $\mathbf P$ and let \(\mathcal E\in\mathcal E(\sigma)\) be an arbitrary possible world corresponding to \(\sigma\), then the translation \(\mathrm{Prob}(\mathbf P)\) introduces for each annotated disjunction in \(\mathbf P\) a sequence of auxiliary atoms and random facts whose semantic effect is to pick exactly one head option (or none) according to the annotated probabilities. A selection \(\sigma\) corresponds to fixing which auxiliary random facts are true in a possible world \(\mathcal E\) of \(\mathrm{Prob}(\mathbf P)\); conversely \(\mathcal E\) determines \(\sigma\). 

Performing the syntactic removals and additions that implement \(\mathrm{fix}(X:=x)\) on \(\mathbf P^\sigma\) is equivalent to performing the corresponding removals and additions on the \(\mathrm{Prob}(\mathbf P)\) translation and then conjoining the choices represented by \(\mathcal E\). Intuitively, the removals delete the same head-options in both syntaxes, and the deterministic additions become deterministic facts in the ProbLog encoding (or deterministic clauses whose head is an ordinary atom).

Because logic programs are modular, if two programs differ only in clauses for a restricted set of atoms (here, the atoms in \(X\) and the auxiliary atoms that were introduced to encode choices for those rules), then the truth of a \(\mathbf P\)-formula that mentions only other atoms remains unaffected. 
In particular, when we compare \(\mathbf P^\sigma_{\text{fix}(X:=x)}\) to \(\mathrm{LP}(\mathrm{Prob}(\mathbf P)_{\text{fix}(X:=x)}) \cup \mathcal E\), all choices about auxiliary atoms (those in \(\mathcal E\)) are already fixed by \(\mathcal E\), and the forced/blocked heads for atoms in \(X\) are syntactically identical in the two programs. It follows that the two programs produce the same derivations for all atoms and thus agree on the truth of every \(\mathfrak P\)-formula.

For (2), we analogously apply Theorem B.2 but start with a selection \(\sigma\) of \(\mathrm{LPAD}(\mathcal P)\). The construction of \(\mathrm{LPAD}(\mathbf P)^\sigma\) and the effect of \(\mathrm{fix}(X:=x)\) on it correspond, under the ProbLog translation, to \(\mathrm{LP}(\mathrm{Prob}(\mathbf P)_{\text{fix}(X:=x)})\) conjoined with the possible world \(\mathcal E(\sigma)\) that encodes the selection \(\sigma\). The same modularity and syntactic-translation observations as above show the two programs are extensionally identical (on atoms of interest) and therefore produce the same answers to every \(\mathfrak P\)-formula.
\noindent\(\square\)
\end{proof}

From here, our proofs for Theorem 4 and 5 follows Kiesel et al. (2023) again. CP-logic estalishes a causal semantics for LPAD-programs. The semantics focus on $\mathfrak{P}$-processes and tying them to the logic of LPAD. More precisely, a $\mathfrak{P}$-process $\mathcal{T}$ is a tuple $(T,\mathcal{I})$, where $T$ is a directed tree and each edge is a labelled with a probability that describes transitioning from each node which are Humean events and. For each non-leaf node, the outgoing edges probabilities must sum to one. $\mathcal{I}$ is a map that assigns each node $n$ in $T$ a Herbrand Interpretation $\mathcal{I}(n)$ in $\mathfrak{B}$. 

Furthermore, for each $n$ in $T$, we associate the probability $\pi^T(n)$ which is given by the product of the probabilities of all edges along the walk from root $\bot$ of $T$ to $n$. This produces a distribution $\pi^T$ on the Herbrand interpretations of $I$ of $\mathfrak{B}$ by 
\[
\pi^\mathcal{T}(I) := \sum_{l \text{ leaf of } T, \, \mathcal{I}(l)=I} \pi^\mathcal{T}(l).
\]

Vennekens et al. (2009) connects LPAD-programs to $\mathfrak{P}$-processes by fixing a LPAD-program $\mathcal{P}$ and defining a hypothetical derivation sequence of $n$ in $\mathcal{T}$ as a sequence of three valued interpretations $(\nu_i)_{0\leq i \leq n}$ where $\nu_0$ assigns False to all atoms not in $\mathcal{I}(n)$ and for $i>0$, there exists $RC \in \mathbf P$ and $j \in [1,l]$ with $\text{body}(RC)^{\nu_i} \neq False$, with $h_j^{i+1} = \text{Undefined}$, and with $\nu_i(p) = \nu_{i+1}(p)$ for all other proposition $p \in \mathcal{P}$.

Such a sequence is terminal if it cannot be extended and each terminal hypothetical derivation sequence $n$ has the same limit $\nu_n$, which is known as the potential in $n$.

For $RC \in \mathcal{P}$, we say that $RC$ fires in a node $n$ of $T$ if for each $1 < i < l(RC)$ there exists a child $n_i$ of $n$ such that $\mathcal{I}(n_i) = \mathcal{I}(n) \cup \{h_i(RC)\}$ and such that each edge $(n,n_i)$ is labeled with $\pi_i(RC)$. Moreover, there exists a child $n_{l(RC)+1}$ of $n$ with $\mathcal{I}(n_{l(RC)+1}) = \mathcal{I}(n)$.  

Let $\mathcal{R}_{\mathcal{E}}(n)$ denotes the set of all rules $RC \in \mathcal{P}$, for which there exists no ancestor $a$ of $n$ with $\mathcal{E}(a) = RC$. Then,
$\mathcal{T}$ then may be an execution model of $\mathcal{P}$, $\mathcal{T} \models \mathcal{P}$, if there exists a mapping $\mathcal{E}$ from the non-leaf nodes of $T$ to $\mathcal{P}$ such that:
\begin{enumerate}
\item $\mathcal{I}(\bot) = \emptyset$ for the root $\bot$ of $T$.
\item In each non-leaf node $n$ a LPAD-clause $\mathcal{E}(n) \in \mathcal{R}_{\mathcal{E}}(n)$ fires with $\mathcal{I}(n) \models \text{body}(\mathcal{E}(n))$.
\item For each leaf $l$ of $T$ there exists no LPAD-clauses $RC \in \mathcal{R}_{\mathcal{E}}(l)$ with $\mathcal{I}(l) \models \text{body}(RC)$.
\item For every node $n$ of $T$ we find body$(\mathcal{E}(n))^{\nu_n} \neq Undefined$, where $\nu_n$ is the potential in $n$.
\end{enumerate}
If $\mathcal{T} \models \mathcal{P}$, then the probability distribution defined by $\pi^{CP}_{\mathcal{P}} := \pi^{\mathcal{T}}$ matches the distribution semantics $\pi^{dist}_{\mathcal{P}}$ and implies the following.

\begin{lemma}[Vennekens et al. (2009), §A.2]
Let $l$ be a leaf node in an execution model $\mathcal{T}$ of the LPAD-program $\mathcal{P}$. In this case, there exists a unique path $p$ from the root $\bot$ of $T$ to $l$. Define the selection $\sigma(l)$ by setting $\sigma(l)(RC) := i \in \mathbb{N}$ if and only if there exists a node $n_j$ along $p$ with $\mathcal{E}(n_j) = RC$ and $\mathcal{I}(n_{j+1}) := \mathcal{I}(n_j) \cup \{h_i(RC)\}$. Otherwise, we set $\sigma(l)(RC) := \bot$. In this way, we obtain that $\mathcal{P}^{\sigma(l)} \models \mathcal{I}(l)$. On the other hand, we find for each selection $\sigma$ of $\mathcal{P}$ a leaf $l$ of $T$ with $\sigma(l) = \sigma$.
\end{lemma}

To finally demonstrate the equivalence of this treatment of counterfactuals to CP-logic, consider the presentation of interventions and counterfactuals in CP-logic from Vennekens et al. (2010).

\begin{algorithm}[H]
\caption{Treatment of Counterfactuals in CP-logic}
\label{alg:cp-counterfactuals}

\begin{algorithmic}[1]
\Require $\mathbf{X}:=\mathbf{x}, \mathbf{E}=\mathbf{e} \subseteq \mathcal{P}$, $\mathfrak{P}$-formula $\phi$.
\Ensure Counterfactual probability $\pi^{CP}_{\mathcal{P}}(\phi \mid E = e, do(X := x))$.
\State Choose an execution model $\mathcal{T}$ of $\mathcal{P}$.
\For{each leaf $l$ of $\mathcal{T}$}
    \State Intervene in the logic program $\mathcal{P}^{\sigma(l)}$ according to $X := x$ to obtain $\mathcal{P}^{\sigma(l),do(X:=x)}$.
    \State Define
    \[
        \pi^l(\phi) :=
        \begin{cases}
            1, & \mathcal{I}(l) \models (E = e) \;\wedge\; \mathcal{P}^{\sigma(l),do(X:=x)} \models \phi \\
            0, & \text{else}.
        \end{cases}
    \]
\EndFor
\State Return
\[
\pi^{CP}_{\mathcal{P}}(\phi \mid E = e, do(X := x)) :=
\sum_{l \text{ leaf of } \mathcal{T}} 
\pi^l(\phi) \cdot \pi^{CP}_{\mathcal{P}}(\mathcal{I}(l) \mid E = e).
\tag{B1}
\]
\end{algorithmic}
\end{algorithm}

As we have established in Lemma \ref{lem:fix_cp}, we can analogously posit that with the rewiring of clauses under the $fix$-operator the following algorithm:

\begin{algorithm}[H]
\caption{Fixed Operator Counterfactuals in CP-logic}
\label{alg:cp-fix_counterfactuals}

\begin{algorithmic}[1]
\Require $\mathbf{X}:=\mathbf{x}, \mathbf{E}=\mathbf{e} \subseteq \mathcal{P}$, $\mathfrak{P}$-formula $\phi$.
\Ensure Counterfactual probability $\pi^{CP}_{\mathcal{P}}(\phi \mid E = e, fix(X := x))$.
\State Choose an execution model $\mathcal{T}$ of $\mathcal{P}$.
\For{each leaf $l$ of $\mathcal{T}$}
    \State Intervene in the logic program $\mathcal{P}^{\sigma(l)}$ according to $X := x$ to obtain $\mathcal{P}^{\sigma(l),fix(X:=x)}$.
    \State Define
    \[
        \pi^l(\phi) :=
        \begin{cases}
            1, & \mathcal{I}(l) \models (E = e) \;\wedge\; \mathcal{P}^{\sigma(l),fix(X:=x)} \models \phi \\
            0, & \text{else}.
        \end{cases}
    \]
\EndFor
\State Return
\[
\pi^{CP}_{\mathcal{P}}(\phi \mid E = e, fix(X := x)) :=
\sum_{l \text{ leaf of } \mathcal{T}} 
\pi^l(\phi) \cdot \pi^{CP}_{\mathcal{P}}(\mathcal{I}(l) \mid E = e).
\tag{B2}
\]
\end{algorithmic}
\end{algorithm}

With these preparations we can now turn to the proof of the desired consistency results:

\begin{proof}[Proof of Theorem \ref{thrm:swip_cp1}]
By Theorem 6, Lemma 9 and Lemma 8 the right-hand side of (B2) for $\mathcal{P}$ is the sum of the conditional probabilities $\pi(\mathcal{E}|E=e)$ of all possible worlds $\mathcal{E}$ of Prob$(\mathcal{P})$ such that
\[
\mathcal{M}\big(\mathcal{E}, \text{LP}(\text{Prob}(\mathcal{P})^{do(X:=x)})\big) \models \phi 
\quad \text{and} \quad
\mathcal{M}(\mathcal{E}, \text{LP}(\text{Prob}(\mathcal{P}))) \models (E=e).
\]
These are exactly the possible worlds that make the query $\phi$ true after intervention while the observation $E=e$ is true before intervening. Hence, we can consult the proof of Theorem 3 to see that (B1) computes the same value as Algorithm \ref{alg:cp-fix_counterfactuals}.
\end{proof}

\begin{proof}[Proof of Theorem \ref{thrm:swip_cp2}]
By Theorem 7, Lemma 8 and Lemma 9 the right-hand side of (B2) for LPAD$(\mathcal{P})$ is the sum of the conditional probabilities $\pi(\mathcal{E}|E=e)$ of all possible worlds $\mathcal{E}$ of $\mathcal{P}$ such that
\[
\mathcal{M}\big(\mathcal{E}, \text{LP}(\mathcal{P}^{do(X:=x)})\big) \models \phi 
\quad \text{and} \quad
\mathcal{M}(\mathcal{E}, \text{LP}(\mathcal{P})) \models (E=e).
\]
These are exactly the possible worlds that make the query $\phi$ true after intervention while the observation $E=e$ is true before intervening. Hence, we can consult the proof of Theorem 3 to see that (B1) computes the same value as Algorithm \ref{alg:cp-fix_counterfactuals}.
\end{proof}

\section{Proof of Theorem 5.1}
\begin{proof}
We summarize notation from Kiesel et al. (2023) \cite{Kiesel.etal2023}. First, we define two propositional alphabets $\mathfrak{P}^S$ to handle the evidence from the source world and 
$\mathfrak{P}^I$ to handle the interventions in the counterfactual. In particular, we set
\[
e(\mathfrak{P}^S) \;=\; e(\mathfrak{P})
\quad\text{and}\quad
i(\mathfrak{P}^I) \;=\; i(\mathfrak{P}),
\]
and we require that
\[
\mathfrak{P}^S \;\cap\; \mathfrak{P}^I \;=\; \emptyset
\]
In this way, we obtain maps
\[
e/i : \; \mathfrak{P} \;\to\; \mathfrak{P}^{\,S/I}
\quad\text{by}\quad
p \;\mapsto\; p^{\,S/I},
\]
that easily generalize to literals, clauses, and programs.

Furthermore, we define the \emph{counterfactual semantics} of $\mathcal{P}$ by
\[
\mathcal{P}^K \;=\; \mathcal{P}^S \;\cup\; \mathcal{P}^I.
\]
Next, we intervene in $\mathcal{P}^K$ for the counterfactual query is for the probability of a consequent $\phi^I$ in the interventional world given evidence $\textbf{E}^S=\textbf{e}$ in the source world, after applying an intervention $fix(\textbf{X}:=\textbf{x})$ to the interventional part of the program. Let 
$P^{K,fix(\textbf{X}:=\textbf{x})}$ denote this modified program given by Algorithm \ref{alg:swip_transform}. Finally, we obtain the desired probability 
\(
\pi^{\mathrm{SCM}}_{P}(\,\phi\mid \textbf{E} = \textbf{e},\; fix(\textbf{X}:=\textbf{x})\,)
\)
by querying the program $P^{K,fix(\textbf{X}:=\textbf{x})}$ for the conditional probability 
$\pi(\phi\mid \textbf{E} = \textbf{e})$.

% \begin{align*}
% \pi_{P}^{SCM}(\phi|\bf E = \bf e, fix(\bf X - \bf x)) &= \pi_{P^{K, fix(\bf X^i = \bf x)}}(\phi^i|\bf E^e = \bf e) \\
% &= \frac{\pi^{SCM}_{P^{K, fix(\bf X^i = \bf x)}}(\phi^i,\bf E^e = \bf e)}{\pi^{SCM}_{P^{K, fix(\bf X^i = \bf x)}}(\bf E^e = \bf e)} \\
% &= \frac{1}{\pi^{SCM}_{P^{K, fix(\bf X^i = \bf x)}}(\bf E^e = \bf e)} \sum_{
%     \substack{
%         &\varepsilon \, possible \,world \\
%         &M(\varepsilon, P^{K, fix(\bf X^i = \bf x)}) \models\phi^i \\
%         &M(\varepsilon, P^{K, fix(\bf X^i = \bf x)}) \models (\bf E^e = \bf e) 
%     } }         \pi_{P^{K,fix(\bf X^i = \bf x}} (\varepsilon)\\
% &= \frac{1}{\pi^{SCM}_{P}(\bf E = \bf e)} \sum_{
%     \substack{
%         &\varepsilon \, possible \,world \\
%         &M(\varepsilon, P^{fix(\bf X = \bf x)}) \models\phi \\
%         &M(\varepsilon, P) \models (\bf E^e = \bf e) 
%     } }    \pi^{SCM}_{P^K, fix(\bf X^i = \bf x} (\varepsilon)\\
% &= \sum_{
%     \substack{
%         &\varepsilon \, possible \,world \\
%         &M(\varepsilon, P^{K, fix(\bf X= \bf x)}) \models\phi \\
%     } }         \pi^{SCM}_{P} (\varepsilon|\bf E = \bf e) \\
% &= \pi_{SCM(P)}(\phi|\bf E = \bf e, \textit{fix}(\bf X = \bf x))
% \end{align*}
\begin{align*}
\pi_{\mathcal{P}}^{SCM}(\phi \mid \mathbf{E} = \mathbf{e},\, fix(\mathbf{X} = \mathbf{x})) 
&:= \pi_{\mathcal{P}^K, fix(\mathbf{X}:=\mathbf{x})}(\phi^I \mid \mathbf{E}^S = \mathbf{e}) \\
&= \frac{\pi_{\mathcal{P}^K, fix(\mathbf{X}:=\mathbf{x})}(\phi^I \land \mathbf{E}^S = \mathbf{e})}{\pi_{\mathcal{P}^K, fix(\mathbf{X}:=\mathbf{x})}(\mathbf{E}^S = \mathbf{e})} \\
&= \frac{\pi_{\mathcal{P}^K, fix(\mathbf{X}:=\mathbf{x})}(\phi^I \land \mathbf{E}^S = \mathbf{e})}{\pi_{\mathcal{P}^K}(\mathbf{E}^S = \mathbf{e})} \\
&= \frac{1}{\pi_{\mathcal{P}^K}(\mathbf{E}^S = \mathbf{e})} \sum_{\varepsilon} \pi_{\mathcal{P}^K}(\varepsilon) \cdot [\varepsilon \models_{\mathcal{P}^K_{fix(\mathbf{X}:=\mathbf{x})}} (\phi^I \land \mathbf{E}^S = \mathbf{e})] \\
&= \frac{1}{\pi_{\mathcal{P}}(\mathbf{E} = \mathbf{e})} \sum_{\varepsilon} \pi_{\mathcal{P}}(\varepsilon) \cdot [\varepsilon \models_{\mathcal{P}^K} \mathbf{E}^S = \mathbf{e}] \cdot [\varepsilon \models_{\mathcal{P}^K_{fix(\mathbf{X}:=\mathbf{x})}} \phi^I] \\
&= \frac{\sum_{\varepsilon} \pi_{\mathcal{P}}(\varepsilon) \cdot [\varepsilon \models \mathbf{E} = \mathbf{e}] \cdot [\varepsilon \models_{\mathcal{M}_x} \phi]}{\sum_{\varepsilon} \pi_{\mathcal{P}}(\varepsilon) \cdot [\varepsilon \models \mathbf{E} = \mathbf{e}]} \\
&= \frac{P_{SCM}(\phi_{\mathbf{X} \leftarrow \mathbf{x}}, \mathbf{E} = \mathbf{e})}{P_{SCM}(\mathbf{E} = \mathbf{e})} \\
&= P_{SCM}(\phi \mid \mathbf{E} = \mathbf{e},\, fix(\mathbf{X} := \mathbf{x})) 
\end{align*}
\end{proof}

\end{document}